\documentclass[11pt]{article}
\pdfoutput=1

\usepackage[utf8]{inputenc} 
\usepackage[T1]{fontenc}    
\usepackage{hyperref}       
\usepackage{url}            
\usepackage{booktabs}       
\usepackage{amsfonts}       
\usepackage{nicefrac}       
\usepackage{microtype}      
\usepackage{xcolor}         
\usepackage{amsthm} 
\usepackage{amsmath}
\usepackage{bm} 
\usepackage{mathtools} 
\usepackage{amssymb}
\usepackage{graphics}
\usepackage{authblk}
\usepackage{caption}
\usepackage{algorithm}
\usepackage{algorithmic}
\usepackage{soul} 
\usepackage{subcaption}
\usepackage{natbib}
\usepackage{fullpage,times}

\hypersetup{
    colorlinks = true,
    citecolor = blue,
    linkcolor = blue
}

\title{Meta Representation Learning with Contextual Linear Bandits}
\title{Meta-Learning Representations with Contextual Linear Bandits}

\author[1]{Leonardo Cella \thanks{leonardocella@gmail.com}}
\author[2]{Karim Lounici \thanks{karim.lounici@polytechnique.edu}}
\author[1,3]{Massimiliano Pontil \thanks{massimiliano.pontil@iit.it}}

\affil[1]{Computational Statistics and Machine Learning, Italian Institute of Technology, Italy}
\affil[2]{CMAP, Ecole Polytechnique, France}
\affil[3]{Department of Computer Science, University College London, U.K.}

\newcommand{\indep}{\perp \!\!\! \perp}

\newcommand{\alphabf}{\bm{\alpha}}
\newcommand{\alphahat}{\widehat{\bm{\alpha}}}

\newcommand{\A}{\mathbf{A}}
\newcommand{\Acal}{\mathcal{A}}
\newcommand{\Abar}{\overline{\mathbf{A}}}

\newcommand{\Bbf}{\mathbf{B}}

\newcommand{\bbf}{\mathbf{b}}
\newcommand{\B}{\mathbf{B}}
\newcommand{\Bbar}{\overline{\mathbf{B}}}
\newcommand{\Bcal}{\mathcal{B}}

\newcommand{\Bhat}{\widehat{\mathbf{B}}}

\newcommand{\Ccal}{\mathcal{C}}

\newcommand{\Dcal}{\mathcal{D}}
\newcommand{\Deltabf}{\mathbf{\Delta}}
\newcommand{\e}{\mathbf{e}}
\newcommand{\E}{\mathbb{E}}

\newcommand{\etab}{\pmb{\eta}}

\newcommand{\Fbar}{\overline{\mathcal{F}}}
\newcommand{\Fcal}{\mathcal{F}}

\newcommand{\h}{\mathbf{h}}
\newcommand{\Hbf}{\mathbf{H}}

\newcommand{\I}{\mathbf{I}}
\newcommand{\Ind}{\mathbb{I}}

\newcommand{\lambdamin}{\lambda_{\min}}
\newcommand{\lambdamax}{\lambda_{\max}}

\newcommand{\M}{\mathbf{M}}
\newcommand{\Mtilde}{\widetilde{\M}}

\newcommand{\bigO}{\mathcal{O}}

\newcommand{\Pbf}{\mathbf{P}}
\newcommand{\Phat}{\widehat{\Pbf}}

\newcommand{\reals}{\mathbb{R}}
\newcommand{\rhat}{\widehat{r}}

\newcommand{\Shat}{\widehat{S}}
\newcommand{\Sigmabf}{\mathbf{\Sigma}}
\newcommand{\Sigmahat}{\widehat{\Sigmabf}}
\newcommand{\Sigmabar}{\overline{\Sigmabf}}
\newcommand{\Sigmabarhat}{\widehat{\overline{\Sigmabf}}}

\newcommand{\Sb}{\mathbb{S}}
\newcommand{\Scal}{\mathcal{S}}

\newcommand{\ubf}{\mathbf{u}}
\newcommand{\uhat}{\widehat{\ubf}}

\newcommand{\bu}{\mathbf{u}}
\newcommand{\bv}{\mathbf{v}}

\newcommand{\vbf}{\mathbf{v}}
\newcommand{\V}{\mathbf{V}}

\newcommand{\vhat}{\widehat{\vbf}}

\newcommand{\Vhat}{\widehat{\mathbf{V}}}

\newcommand{\w}{\mathbf{w}}

\newcommand{\what}{\mathbf{\widehat{w}}}
\newcommand{\W}{\mathbf{W}}
\newcommand{\What}{\widehat{\W}}

\newcommand{\Wbar}{\overline{\mathbf{W}}}
\newcommand{\x}{\mathbf{x}}

\newcommand{\X}{\mathbf{X}}
\newcommand{\Xcal}{\mathcal{X}}

\newcommand{\xbar}{\overline{\mathbf{x}}}

\newcommand{\Y}{\mathbf{Y}}

\newcommand{\Nold}{N^{\rm tr}}
\newcommand{\ntrain}{\Nold}
\newcommand{\Abf}{\mathbf{A}}

\newtheorem{lemma}{Lemma}
\newtheorem{assumption}{Assumption}
\newtheorem{theorem}{Theorem}
\newtheorem{corollary}{Corollary}

\newtheorem{definition}{Definition}

\newcommand{\norm}[1]{\left\lVert#1\right\rVert}
\newcommand{\opnorm}[1]{\left\lVert#1\right\rVert_{\rm op}}

\newcommand{\nucnorm}[1]{\left\lVert#1\right\rVert_{\ast}}

\DeclareMathOperator*{\argmax}{\ensuremath{argmax}}

\DeclareMathOperator*{\argmin}{\ensuremath{argmin}}

\usepackage[most]{tcolorbox}
\newtcolorbox[auto counter,number within=section]{protocol}[1][]{
  enhanced,
  breakable,
  fonttitle=\scshape,
  title={Protocol \thetcbcounter},
  #1
}

\usepackage[textwidth=2.0cm, textsize=tiny]{todonotes} 

\begin{document}

\maketitle

\begin{abstract}
Meta-learning seeks to build algorithms that rapidly learn how to solve new learning problems based on previous experience.
In this paper we investigate meta-learning in the setting of stochastic linear bandit tasks. We assume that the tasks share a low dimensional representation, which has been partially acquired from previous learning tasks. 
We aim to leverage this information in order to learn a new downstream bandit task, which 
shares the same representation.
Our principal contribution is to show 
that if the learned representation estimates well the unknown one, then the downstream task can be efficiently learned by a greedy policy that we propose in this work. We derive an upper bound on the regret of this policy, which is, up to logarithmic factors, of order $r\sqrt{N}(1\vee \sqrt{d/T})$, where $N$ is the horizon of the downstream task, $T$ is the number of training tasks, $d$ the ambient dimension and $r \ll d$ the dimension of the representation. We highlight that our strategy does not need to know $r$. We note that if $T> d$ our bound achieves the same rate of optimal minimax bandit algorithms using the true underlying representation. 
Our analysis is inspired and builds in part upon previous work on meta-learning in the i.i.d. full information setting \citep{tripuraneni2021provable,boursier2022trace}. As a separate contribution we show how to relax certain assumptions in those works, thereby improving their representation learning and risk analysis.
\end{abstract}


\section{Introduction}

In the last years the problem
of meta-learning and transferring knowledge between a set of learning tasks has emerged as a functional area of machine learning. While a lot of work has been done in the full information i.i.d. statistical setting \citep{baxter2000model,khodak2019adaptive,pontil2013excess,maurer2016benefit,denevi2018incremental,denevi19a,tripuraneni2021provable,boursier2022trace}, investigation of meta-learning in the partial information interactive setting are lacking and we are only aware of few recent works \citep{cella2020,yang2022nearly}. More work has been done on the multitask bandit setting \citep{kveton2021meta,hu2021near,yang2020impact,cellaSparseMTL}, 
however the goal there is different, in that we wish to learn well a prescribed finite set of tasks as opposed to leverage knowledge from these to learn a novel downstream task.
 

In this paper we focus on meta-representation learning with contextual linear bandit tasks. 
We assume that the tasks share a low dimensional representation, which has been partially acquired from previously observed (interactive or bandit) learning tasks and wish to leverage
this information in order to learn a new downstream bandit task, sharing the same representation. The interactive, non i.i.d. nature of the tasks, induces a trade-off between exploration and exploitation and present several challenges to derive learning guarantees. 
Even in the statistical setting, it was recently pointed out in \cite{tripuraneni2021provable,boursier2022trace} that, despite a long line of works on meta-representation learning, the theoretical understanding of the underlying phenomena is incomplete, and sharper learning bounds on meta-representation learning can be derived when additional assumptions on the problem are made. 
When considering contextual linear bandits, the theoretical analysis becomes even more complex. To the best of our knowledge, we are the first to consider this setting, obtaining promising results even when the number of interactions is 
smaller than the ambient dimension. Our analysis is inspired by the aforementioned works. Surprisingly, if compared to \citep{boursier2022trace} we obtain a less computation demanding estimator under more relaxed assumptions. In particular, this yields a minimax bandit greedy algorithm.

\paragraph{Contributions.} We make three main contributions. Firstly, we investigate the benefit of meta-representation learning relying on a set of $T$ training tasks sharing the same representation of the current bandit task (Theorem \ref{Th:RegretUBound}). Our results are inspired and built upon \citep{tripuraneni2021provable,boursier2022trace} 
but we notably remove their invertibility assumption on the input covariance matrix, so that our analysis can cover more general high-dimensional settings. 
Moreover, we propose a more practical representation estimator which does not need to know the representation dimension. 
 Secondly, we address the scenario in which the past information corresponds to completed contextual bandit tasks and we give a bound on estimation quality of the representation~(Theorems \ref{Thm:NewBias1} and \ref{Th:NewBias2}). To this end we extend the arguments of \cite{tripuraneni2021provable} to the bandit setting where samples are not i.i.d. anymore by exploiting martingale tools. Thirdly, we highlight that our regret bound (Corollary \ref{cor:main}) matches the optimal minimax lower bound and we do not impose any norm boundedness assumption on the arms, unlike typically done in the bandit literature. Finally, we remark that our results also hold in the i.i.d. statistical setting, proving significant improvements over the analysis in the aforementioned works, both in terms of computational efficiency and assumptions.

\paragraph{Literature Review.} 
If focusing on meta-learning a representation in the standard supervised learning setting the most relevant papers are \citep{tripuraneni2021provable,boursier2022trace}. The latter work adopts a truncated-SVD approach to estimate the shared representation {requiring the knowledge of the rank of the representation}. On the other side, the former relies on the method-of-moments. Empirically, it has been observed that the truncated-SVD solution outperforms the one based on the method-of-moments \citep{boursier2022trace}. In addition, it also works in poor data regimes, for any number of samples. \\
In the bandits literature the closest works to our are \citep{cella2022multi} and \citep{yang2022nearly}.The former work investigates the multitask bandit setting, presenting a greedy policy based on trace-norm regularization and providing a regret bound on the set of training tasks. The latter work considers the lifelong learning (or meta-learning) framework as we do in this work and further analyzes the multi-task learning setting with infinite arms. They propose a specific variant of the explore-then-commit policy. However, differently from the standard bandit literature \citep{abbasi2011improved,lattimore2020bandit} they assume the noise to be i.i.d. and distributed as a standard Gaussian random variable. Therefore, it is not clear how the noise variance impacts the obtained theoretical guarantees and how the proposed strategy handles possible different values. They also assume to know the rank parameter $r$.


\paragraph{Notation.} We denote vectors by bold lowercases and matrix by bold uppercases. Given a vector $\x\in\reals^d$ we denote with $\norm{\x}$ its Euclidean norm. For any matrix $\W\in\reals^{d\times T}$ we use $\norm{\W}_F$ for its Frobenious norm (Euclidean norm of its singular values), $\norm{\W}_{\rm op}=\sigma_{\max}(\W)$ for its operator norm (largest singular value) and denote its rank by $\rho(\A)$. Given a pair of matrices $\A,\B\in\reals^{d\times d}$, the expression $\A\succeq\B$ means that $\A-\B$ is positive semi-definite. Given a square matrix $\Sigmabf$, we use $\lambda_i(\A)$ to denote its $i$-th largest eigenvalue. Analogously, we respectively use $\lambda_{\min}(\Sigmabf)$ and $\lambda_{\max}(\Sigmabf)$ for the minimum and the maximum eigenvalues. A similar notation holds for singular values, we denote them by $\sigma_i(\A), \sigma_{\max}(\A)$ and $\sigma_{\min}(\A)$. 
Finally, for any positive integer $T$, we denote the corresponding set of integers $\{1,\dots,T\}$ by $[T]$.

\section{Problem Setting}
In this section, we 
first recall the linear contextual bandit setting and then introduce the meta-learning framework with bandit tasks.

\subsection{Linear Represented Contextual Bandits}
Let $N$ and $K$ be positive integers. 
A contextual linear bandit problem is defined as a sequence of $N$ interactions between a learning agent  and the environment. At each interaction  $n\in[N]$, the agent is given a decision set $\Dcal_n = \{\x_{n,1}, \dots, \x_{n,K}\} \subseteq \reals^d$ from which it has to pick one arm $\mathbf{x}_{n}\in\Dcal_n$ among the $K$ available ones. Subsequently, the agent observes the corresponding reward $y_n$ which is assumed to be a noisy linear regression with respect to an unknown vector parameter $\w\in\reals^d$. In this work our focus is on linear contextual bandits sharing an unknown representation $\B\in\reals^{d\times r}=(\bbf_1,\dots,\bbf_r)$ with $r$ orthonormal columns and $d>r$. That is, we consider that $\w=\B \alphabf$, for some $\alphabf \in \mathbb{R}^r$, so that 
the observed reward satisfies 
\begin{equation}\label{Eq:LinearReward}
  y_{n} = \x_n^\top \B \alphabf + \eta_n,
  = \x_n^\top \w + \eta_n,
\end{equation}
where $\eta_n$ is a noise random variable which we specify below. Relying on the knowledge of the true regression vector $\alphabf$ and the representation matrix $\B$, at each round $n\in[N]$ the optimal policy  
selects $\mathbf{x}^*_n = \arg\max_{\mathbf{x}\in \Dcal_n} \mathbf{x}^\top \B \alphabf$, maximizing the instantaneous expected reward. The learning objective is then to minimize the \textit{pseudo-regret} incurred with respect to the optimal policy,
    \begin{equation}
    	R(N,\alphabf) = \sum_{n=1}^{N}(\mathbf{x}^*_n - \mathbf{x}_{n})^\top \B \alphabf. \nonumber
    \end{equation}
\noindent For the meta-learning setting discussed in this work, we use a relaxation of the arm assumption considered in \cite{cella2022multi} for the multi-task case, the regression vector and the noise variables. Note the standard literature on linear contextual bandits \citep{auer2002using,abbasi2011improved,lattimore2020bandit} impose a restriction on the norm of the context arms $\norm{\x}\leq L$ for some absolute constant $L>0$. We do not impose any such condition.

\begin{assumption}[Feature set and parameter]\label{Ass:BoundedNorms} At each round $n\in[N]$, the decision set $\Dcal_{n}\in\reals^d$ consists of $K$ $d$-dimensional vectors $\x_k$ admitting representation $\x_k = \Sigma_k^{1/2}\mathbf{z}_k$ where $\Sigma_k$ is the covariance operator of $\x_k$ and $\mathbf{z}_k$ is a $d$-dimensional vector with independent subGaussian entries admitting zero mean and variance $1$. We set $\max_{1\leq k \leq K}\left\lbrace \|\mathbf{z}_k\|_{\psi_2}\right\rbrace <C_{\mathbf{z}}$ for some absolute constant $C_{\mathbf{z}}>0$. Specifically, we assume the tuples $\Dcal_{1},\dots,\Dcal_{N}$ to be drawn i.i.d. from an fixed by unknown zero mean subGaussian joint distribution  $p$ on $\mathbb{R}^{Kd}$. We denote by $C_{\x}$ the subGaussian Orlicz norm of the $K$ marginal distributions of $p$ corresponding to the $K$ arms, that is 
\[
C_{\x} := \max_{1\leq k \leq K}\left\lbrace \|\x_k\|_{\psi_2}\right\rbrace \leq  C_{\mathbf{z}} \max_{1\leq k \leq K}\left\lbrace \opnorm{\Sigma_k}^{1/2} \right\rbrace <\infty.
\]
Finally, coherently with the standard bandit literature \citep{abbasi2011improved,lattimore2020bandit}, we assume that $\norm{\w}\leq  L$, for some constant $L>0$.
\end{assumption}
\noindent 
From the previous assumption it follows that each arm $k$ admits zero mean, square-integrable marginal distribution with covariance operator $\Sigma_k\in\reals^{d\times d}$. In our approach, vectors associated to different arms are allowed to be correlated between each other.\\ Following the recent result of \cite{cella2022multi} for the multi-task setting, and the works of \cite{oh2020sparsity,bastani2021mostly} for sparse  high-dimensional bandits, we introduce the following assumption on the arms distribution. This mild assumption allows  the design of exploration-free policies which avoid random plays during arms' selection.
\begin{assumption}[Arms distribution]\label{Ass:ArmsDistribution} 
There exists a constant $\nu<\infty$ such that $p(-\xbar)/p(\xbar)\leq\nu$, for every $\xbar \in \reals^{Kd}$. Additionally, let us consider a permutation $(a_1,\dots,a_K)$ of $(1,\dots,K)$. For any integer $i\in\{2,\dots,K{-}1\}$ and any fixed vector $\w\in\reals^d$, there exists a constant $\omega_\Xcal<\infty$ such that
\begin{equation*}
    \E\left[\x_{a_i} \x_{a_i}^\top \Ind\left\{\x_{a_1}^\top \w<{\cdots}<\x_{a_K}^\top\w\right\}\right] \preceq \omega_{\Xcal} \E\left[ \left( \x_{a_1} \x_{a_1}^\top {+} \x_{a_K}\x_{a_K}^\top \right) \Ind\left\{\x_{a_1}^\top \w <\cdots<\x_{a_K}^\top\w\right\}\right].
\end{equation*}
\end{assumption}
\noindent 
Please notice that Assumption \ref{Ass:ArmsDistribution} characterizes a wide class of distributions, both discrete and continuous, e.g. Gaussian, multi-dimensional and uniform. The parameter $\nu$ controls the skewness of the distribution $p$; if $p$ is symmetric then $\nu=1$. The value $\omega_\Xcal$ depends on the arms correlation. The stronger the correlation, the smaller $\omega_\Xcal$ will be. Finally, when arms are generated i.i.d. from a multivariate Gaussian or a multivariate uniform distribution over the sphere we have $\omega_{\Xcal}=O(1)$.

\noindent Let us now define the filtration $(\Fcal_{n})_{n\geq 0}$ on a probability space $(\Omega,\mathcal{A},\mathbb{P})$ as follows: $\Fcal_{0}$ is the trivial  $\sigma$-field $\{\emptyset,\Omega\}$, and for any $n\geq 1$,
$$
\Fcal_{n}=\sigma\left(\x_1,\eta_1,\dots,\x_{n},\eta_n\right).
$$
We can now define the standard subGaussian noise assumption \citep{abbasi2011improved,bastani2020online,cella2020}.
\begin{assumption}[subGaussian noise]\label{Ass:subGaussnoise} 
The noise variables $(\eta_{t,n})_{t,n}$ are a sequence of subGaussian random variables adapted to the filtration $\{\Fcal_{n}\}_{n\geq 0}$ and such that for any $t \in [T]$ and  $n\geq 1$,
$$
\E[\eta_{t,n}|\Fcal_{n-1}] = 0,\quad\text{and}\quad \E[\eta_{t,n}^2|\Fcal_{n-1}] \leq \sigma^2,
$$
and $\eta_1,\ldots,\eta_n$ are mutually independent conditionally on $\Fcal_{n-1}$. We denote by $c_{\eta}$ the subGaussian norm of the $\eta_{n}$'s, that is, $\max_{n\in[N]}\{\|\eta_{n}\|_{\psi_2}\} \leq c_\eta$.
\end{assumption}
To conclude, we introduce the following notation referring to the different considered covariance matrices associated with the sequence of chosen contexts. They will be used throughout the paper.
\paragraph{Covariance Matrices.} We indicate the theoretical $d \times d$ covariance matrix as 

\begin{equation}
\label{eq:S-111}
    \Sigmabf = \frac{1}{K}\E\left[\sum_{k=1}^K \x_k \x_k^\top \right] = \frac{1}{K} \sum_{k=1}^K \Sigma_k,
\end{equation} where the expectation is over the decision set sampling distribution $p$ 
which is assumed to be shared between the tasks. We denote the empirical covariance matrix as
\begin{equation}
\label{eq:S-222}
\Sigmahat_{n}=\frac{1}{n}\sum_{s=1}^n\x_{s} \x_{s}^\top.
\end{equation}

\subsection{Meta-Learning Linear Contextual Bandits}\label{SubSec:LTL}
In this paper we address the meta-learning problem. We assume that we observe a set of $T$ training linear contextual bandit tasks $\w_1,\dots,\w_T$ which
share the same set of low-dimensional features $\B$. That is we have that $\W_T=\B\A_T$ where $\B\in\reals^{d\times r}$ and $\A_T=[\alphabf_1,\dots,\alphabf_T]\in\reals^{r\times T}$ such that $\B^\top\B=\I_r$. After completing each task $t$ we store the whole set of $\Nold$ interactions in a dataset $Z_t = (\x_{t,n}, Y_{t,n})_{n=1}^{\Nold}$, $t \in [T]$. Each such dataset
corresponds to the recording of the learning policy with the environment when solving the corresponding task. {Given the past dataset $Z_1,\dots,Z_T$ corresponding to already completed tasks}, our goal is to minimize the transfer regret $ R({T},{N})$, which is defined as
\begin{equation}\label{Eq:LTLRegret}
    R({N},{\alphabf_{T+1}}) {=} \sum_{n=1}^N (\x_{T+1,n}^* {-} \x_{T+1,n})^\top \B {\alphabf_{T+1}}
\end{equation}
where $N$ denotes the number of rounds for the current task and the we defined the optimal arm.
$\x_{T+1,n}^*=\arg\max_{\x\in\Dcal_{T+1,n}}\x^\top \B {\alphabf_{T+1}}$.

\noindent The approach we consider consists in estimating the representation $\B_T$ shared across the $T$ related tasks, and then use it to minimize the transfer regret (\ref{Eq:LTLRegret}). In Section \ref{Sec:GreedySolution} we propose and analyze a bandit policy which is provided with an estimator of the representation. Then in Section  \ref{Sec:ReprLearning} we discuss two estimators of the representation, based on the recordings $Z_1,\dots,Z_T$.

\section{A Greedy Solution}\label{Sec:GreedySolution}
In the standard meta-learning setting $T$ datasets containing i.i.d. samples for each of the training tasks  are given, see 
\citep{boursier2022trace,denevi19a,maurer2016benefit,tripuraneni2021provable}, from which the learner extracts a representation and uses it as prior knowledge for solving the next task. In this paper, we relax the aforementioned assumption allowing each past task $j$ to be associated with a dataset $Z_j$ containing the whole history of interactions $(\x_{j,n},Y_{j,n})_{n=1}^{\Nold}$ consisting of the sequence of chosen arms and the observed rewards.\\
In this section, let us assume that we are given an estimate $\Bhat_T$ of rank $\tilde{r} = \rho(\Bhat_T)$ of the true representation $\B$. At any interaction $n\in[N]$ of task $T+1$ we will then estimate the true regression vector $\alphabf_{T+1}$ as
\begin{equation}\label{Eq:alphahat}
    \alphahat_{T+1,n} \in \argmin_{\alphahat\in\reals^{
    \tilde{r}}}\sum_{i=1}^n  \left(y_{T+1,i} - \x_{T+1,i}^\top (\Bhat_T \alphahat)\right)^2.
\end{equation}

By design, the greedy policy will then pick the next arm as
\[
    \x_{T+1,n} \in \argmax_{\x\in\Dcal_{T+1,n}} \langle \x, \Bhat_T\alphahat_{T+1,n}\rangle.
\]
We can design a greedy bandit algorithm (whose pseudocode is given in Algorithm \ref{Alg:Greedy}) whose objective is to minimize the transfer regret (\ref{Eq:LTLRegret}) using $\Bhat_T$ as prior information for solving the next task.

\begin{algorithm}[t]\caption{Meta-Represented Greedy Policy}\label{Alg:Greedy}

\begin{algorithmic}[1]
\REQUIRE Confidence parameter $\delta$, noise variance $\sigma^2$
\STATE We are given a representation estimate $\Bhat_T$
\STATE At round $n=1$ arm is picked randomly
\STATE Observe $Y_{T+1,1}$
\FOR{$n \in 2,\dots,N$} 
    \STATE update $\alphahat_{T+1,n}$ according to (\ref{Eq:alphahat}) 
    \STATE observe $\Dcal_{T+1,n}$ satisfying Ass. \ref{Ass:ArmsDistribution}
    \STATE pick $\x_{T+1,n} \in \argmax_{\x\in\Dcal_{T+1,n}} \langle \x, \Bhat_T\alphahat_{T+1,n}\rangle$
    \STATE observe reward $Y_{T+1,n}$ defined as in (\ref{Eq:LinearReward})
\ENDFOR
\end{algorithmic}
\end{algorithm}
\subsection{Bounding the Transfer Regret}

We begin by proving a regret bound on the $T+1$ downstream task, which holds under the assumption 
that, there exists a function $n \mapsto \gamma(T,n)$ and integer $N_0$, such that, whenever $n \geq N_0$, given the estimates of the representation $\Bhat_T$ and the current task vector $\alphahat_{T+1,n}$ we have
\begin{equation}\label{Eq:OracleBound}
    \norm{\Bhat_T\alphahat_{T+1,n} - \B\alphabf_{T+1}} \leq \gamma(T,n).
\end{equation}
In the following section we will derive an expression for the integer $N_0$ and the function $\gamma$ for the specific representation estimators.  We are now ready to prove a high probability bound to the transfer-regret defined in (\ref{Eq:LTLRegret}) in terms of the estimation error bound $\gamma(T,n)$ incurred at any round $n\in[N]$ greater than $N_0$.
\begin{theorem}\label{Th:RegretUBound}Let Assumptions \ref{Ass:BoundedNorms}, \ref{Ass:ArmsDistribution}, \ref{Ass:subGaussnoise} and the condition in \eqref{Eq:OracleBound} be satisfied. Then we have with probability at least $1-\delta$ 
\begin{equation}
   R({N},\alphabf_{T+1}) \leq \bigO\left( \sqrt{\tilde{r}} \bigvee \sqrt{\log(KN\delta^{-1})}\right) \left( N_0 +\sum_{n=N_0+1}^N \gamma(T, n) \right).
\end{equation}
\end{theorem}

If we assume in addition that the number of rounds for the $T$ training tasks is sufficiently large, then we can replace $\tilde{r}$ by $r$ the rank of the representation $\Bbf$. This is done in Corollary \ref{cor:main}. See also Appendix \ref{sec:rankcontrol} for more details.

\begin{proof}
We can start controlling the transfer regret as
\begin{align*}
    R({N},\alphabf_{T+1})  &\leq  \sum_{n=1}^{N_0} L \norm{\Pbf \left(\x^*_{T+1,n} - \x_{T+1, n} \right)} + \sum_{n=N_0}^N \langle \x^*_{T+1,n} - \x_{T+1,n}, \B \alphabf_{T+1} \rangle \\
    &\leq \sum_{n=1}^{N_0} L \max_{1\leq n\leq N}\max_{\x\in \mathcal{D}_{T+1,n}} 2 \norm{\Pbf(\x)} +  \sum_{n=N_0}^N r_{T+1,n}
\end{align*}
where we denoted by $\Pbf$ the projection matrix onto $\Im(\Bhat_T)+\Im(\Bbf)$, where $\Im(\cdot)$ denotes the range on an operator. Note that $\Pbf$ is of rank at most $\rho(\Bhat_T)+\rho(\Bbf):=\tilde{r}$. Then Lemma \ref{lem:norm-arms} (applied only on one arm: the $(T+1)$-th arm) gives for any $\delta\in (0,1)$ with probability at least $1-\delta/4$,
\begin{equation*}
    \max_{n\in[N]}\max_{\x_{T+1,n}\in \mathcal{D}_{T+1,n}} \|\Pbf(\x_{T+1,n})\|_2^2 \leq C \max_{1\leq k \leq K}\{\opnorm{\Sigma_k}\}\left( \tilde{r} \bigvee  \log(8\delta^{-1}KN)\right), 
\end{equation*}
for some constant $C>0$ that depends only on  $C_{\mathbf{z}}$. Indeed, the quantities appearing in Lemma \ref{lem:norm-arms} 
becomes $\mathrm{trace}(\Pbf\Sigma_{k}\Pbf)\leq \max_{1\leq k\leq K}\{\opnorm{\Sigma_k}\}\, \tilde{r}$, $\opnorm{\Pbf\Sigma_{k}\Pbf}\leq \max_{1\leq k\leq K}\{\opnorm{\Sigma_k}\}$ and 
$\|\Pbf\Sigmabf_{k}\Pbf\|_{\rm F}\leq \max_{1\leq k\leq K}\{\opnorm{\Sigma_k}\} \sqrt{\tilde{r}}$ for any $k\in [K]$.  Combining the two last displays we get with probability at least $1-\delta/4$,
\begin{equation}
\label{eq:RTNbound}
    R({N},\alphabf_{T+1})  \leq  C^{1/2} L\sqrt{\max_{1\leq k\leq K}\{\opnorm{\Sigma_k}\}}\left( \sqrt{\tilde r} \bigvee \sqrt{\log(8KN\delta^{-1})}\right) N_0  + \sum_{n=N_0}^N r_{T+1,n}.
\end{equation}

Considering the greedy algorithm displayed in Figure \ref{Alg:Greedy}, its instantaneous regret $r_{T+1,n}$ satisfies
\begin{eqnarray*}
    r_{T+1,n} &\leq & \langle \x_{T+1,n}^*, \B\alphabf_{T+1} - \Bhat_T\alphahat_{T+1,n} \rangle + \langle \x_{T+1,n}, \Bhat_T \alphahat_{T+1,n} - \B\alphabf_{T+1} \rangle + \\ 
    &~&~~~+  \langle \x^*_{T+1,n} - \x_{T+1,n}, \Bhat_T \alphahat_{T+1,n} \rangle \leq 
    \langle \x^*_{T+1,n} - \x_{T+1,n}, \B\alphabf_{T+1} - \Bhat_T\alphahat_{T+1,n} \rangle 
    ,
\end{eqnarray*}
where we used that for any $1\leq n \leq N$, $\langle \x_{T+1,n}^* - \x_{T+1,n}, \Bhat_T \alphahat_{T+1,n} \rangle\leq 0$ a.s. by definition of our policy. The Cauchy-Schwartz inequality combined with \eqref{Eq:OracleBound} then yields 
\begin{align*}
    r_{T+1,n} &\leq
    \langle \x^*_{T+1,n} - \x_{T+1,n}, \B\alphabf_{T+1} - \Bhat_T\alphahat_{T+1,n} \rangle \\
    &\leq \norm{\Pbf\left(\x^*_{T+1,n} - \x_{T+1,n}\right)} \norm{\B\alphabf_{T+1} - \Bhat_T\alphahat_{T+1,n}}\\
    &\leq \norm{\Pbf\left(\x^*_{T+1,n} - \x_{T+1,n}\right)}\gamma(T,n).
\end{align*}

An union bound combining \eqref{eq:RTNbound} with the last three displays (with $x$ replaced by $\log(3\delta^{-1}N)$, $\delta\in (0,1)$) gives, with probability at least $1-\delta$,
\begin{align*}
    R(T,N) \leq C' \left( \sqrt{\tilde{r}} \bigvee \sqrt{\log(8KN\delta^{-1})}\right) \left( N_0 +\sum_{n=N_0+1}^N \gamma(T, n) \right),
\end{align*}
for some constant $C' = C'\left(c_{\eta}, C_{\mathbf{z}},\sigma,\nu ,\omega_\Xcal,\kappa\big(\Sigmabar\big),\max_{k}\left\lbrace\|\Sigma_{k}^{1/2}\|_{\mathrm{op}},L\right\rbrace\right)>0$.
\end{proof}

\section{Learning the Representation}\label{Sec:ReprLearning}
In this section we consider the problem of estimating the hidden representation $\B$ by $\Bhat_T$. Before proposing our solution, we will estimate $\B$ by improving the approach by \citep{boursier2022trace,tripuraneni2021provable} that was originally proposed for the i.i.d. case. Then, in the second half of the section we will introduce our approach. Not only it relaxes the assumptions of \cite{boursier2022trace,tripuraneni2021provable} but also requires less computations and is agnostic to the value of the rank $r$.

Similarly to what has been done in \citep{boursier2022trace}, let us consider $\W_T = \B \A_T \in \reals^{d\times T}$ where $\A_T = [\alphabf_1,\dots,\alphabf_T]^\top$.
As it was also investigated in \citep{cella2022multi}, we consider the trace-norm estimator
\begin{equation}\label{Eq:What}
    \What_T \in \argmin_{\W \in \reals^{d\times T}}\left\{ \frac{1}{\Nold T} \sum_{t,n}^{T,\Nold} \left(y_{t,n} - \langle \x_{t,n}, \w_{t,n} \rangle \right)^2+ \lambda_{\Nold} \norm{\W}_* \right\}.
\end{equation}
Now, we compute the following decomposition of $\What_T$ to obtain $\Bhat_T$:
\begin{align}
\label{eq:Bhatdecompose}
  \What_T=\Bhat_T \widehat{\bf{A}}_T,\quad \Bhat_T\in\reals^{d\times \rhat}\;\;\text{with}\;\; \Bhat_T^\top \Bhat_T = \I_{\rhat},\;\;\text{and}\;\; \widehat{\bf{A}}_T\in\reals^{\rhat\times T}.
\end{align}
Here we set $\rhat = \rho(\What_T)$ equal to the rank of matrix $\What_T$.\\
An alternative choice to estimate the representation is to follow the same approach of \citep{boursier2022trace}, which requires the knowledge of the rank $r$ of $\W$. It then consists in computing the closest rank $r$ approximation of $\What_T$: 
\begin{equation}\label{Eq:Wbar}
    \Wbar_T \in \argmin_{\W \in \reals^{d\times T}: \rho(\W)\leq r} \norm{\W - \What_T}_{\rm F}.
\end{equation}
Next we compute the following decomposition: 
\begin{align}
\label{eq:Bdecompose}
  \Wbar_T=\Bbar_T \Abar_T,\quad \Bbar_T\in\reals^{d\times r}\;\;\text{with}\;\; \Bbar_T^\top \Bbar_T = \I_r,\;\;\text{and}\;\; \Abar_T\in\reals^{r\times T}.
\end{align}
We will then use $\Bbar_T$ as an estimate of the representation $\B$ for the new $T+1$-th task.

Before going on with the analysis we need to recall the restricted strong convexity (RSC) condition, a standard definition from high-dimensional statistics. Notice, that this was used in \cite{cellaSparseMTL} for the multi-tasks stochastic linear bandits. Denote by $\Sigmabar = \mathrm{diag}(\Sigmabf,\ldots,\Sigmabf)\in\reals^{dT\times dT}$.
 
\begin{definition}[RSC($\overline{r}$) Condition]\label{Def:RSC}
Let $\overline{\W}\in \reals^{d\times T}$ admits singular value decomposition $\overline{\W}={\bf U}{\bf D}{\bf V}^\top$ , and let ${\bf U}^{\overline{r}}$ and ${\bf V}^{\overline{r}}$ be the submatrices formed by the top $\overline{r}$ left and right singular vectors, respectively.
We say that the restricted strong convexity RSC($\overline{r}$) condition is met for the theoretical multi-task matrix $\Sigmabar\in\reals^{dT\times dT}$, with positive constant $\kappa(\Sigmabar,\overline{r})$ if
\[
\min\left\{
    \frac{\norm{\rm{Vec}(\Deltabf)}_{\Sigmabar}^2}{2 \norm{\rm{Vec}(\Deltabf)}_2^2}: \Deltabf \in \Ccal(\overline{r})
\right\} \geq \kappa(\Sigmabar,\overline{r}).
\]
We used $\rm{Vec}(\Deltabf)$ for the $\reals^{dT}$ vector obtained by stacking together the columns of $\Deltabf$, and $\Ccal(r)$ is defined as
\begin{equation}\label{Eq:Cset}
    \left\{ \Deltabf \in \reals^{d\times T}:\; \nucnorm{\Pi(\Deltabf)} \leq 3 \nucnorm{\Deltabf - \Pi(\Deltabf)}\right\},
\end{equation}
where $\Pi(\Deltabf)$ is the projection operator onto set $\Bcal:=\{\Deltabf\in\reals^{d\times T}: \rm{Row}(\Deltabf)\perp {\bf U}^{\overline{r}}, \rm{Col}(\Deltabf)\perp {\bf V}^{\overline{r}}\}$.
\end{definition}
We are now ready to present one of the main technical results of our paper. It allows to relax the invertibility assumption of the covariance matrix $\Sigmabf$ in \citep{tripuraneni2021provable,boursier2022trace}. Indeed, as it is proved in Lemma \ref{lem:invertibility} 
in the appendix, we only need $\Sigmabf$ to satisfy the RSC condition. Since the proof is quite technical we leave it in the appendix and write only the statement here.
\begin{lemma}
Let Assumptions \ref{Ass:BoundedNorms} and \ref{Ass:ArmsDistribution} be satisfied. Let $\Sigmabf$ satisfies the RSC($\bar{r}$) condition for some integer $\bar{r}\geq 1$ with constant $\kappa(\Sigmabf)$. Let $\overline{\Bbf} \in \reals^{d\times \bar{r}}$ admits orthonormal column vectors. Then, there exists a numerical constant $C>0$ such that, for any $\delta\in(0,1)$, we have with probability at least $1-\delta$, for any 
$n\geq C \bar{r}\log(2KN\delta^{-1})$
$$
\lambda_{\bar{r}}(\overline{\Bbf}^\top\widehat{\Sigmabf}_{n}\overline{\Bbf})>\frac{\kappa(\Sigmabf)}{4\nu \omega_{\Xcal}}>0.
$$
\end{lemma}
The proof of this result uses the Freedman’s inequality for matrix-martingale \citep{Oliveira2010ConcentrationOT} combined with a truncation argument on the norm of the arms. 




\subsection{Controlling the Estimation Error}
In this section we will discuss how to derive the estimation error bound \eqref{Eq:OracleBound} for $\Bbar_T$ (and $\Bhat_T$). The considered approach follows a bias-variance analysis of the estimation error $\Bbar_T\alphahat_{T+1,n} - \B\alphabf_{T+1}$. The bias part is treated through a novel perturbation argument which is valid even in the high-dimensional setting $d>n$ as it does not requires the invertibility of $\Sigmabf$ contrarily to the prior works \cite{tripuraneni2021provable,boursier2022trace} in the i.i.d. setting. We will remove the subscript $T+1$ when not necessary.


\begin{theorem}[Bias Term]\label{Thm:NewBias1}
Considering a sequence of random covariates $\left(\x_s\right)_{s=1}^n\in\reals^d$ satisfying Assumptions \ref{Ass:BoundedNorms}, \ref{Ass:ArmsDistribution} and \ref{Ass:subGaussnoise} where for any given policy $\pi$, $\x_s$ is $\Fcal_{s-1}$-measurable. For any $\delta\in (0,1)$, let
\begin{align}
\label{eq:N0cond1}
    N_0(x):= C \left(  r\bigvee \log(N\delta^{-1}) \right),
\end{align}
for some large enough numerical constant $C>0$ that do not depend on $n,N,r,d,\delta$. 
Let the RSC($r$) condition be satisfied by $\Sigmabf$, then the matrix $\Bbar$ defined in \eqref{eq:Bdecompose} satisfies, with probability at least $1-\delta$, for any $N_0(\delta)\leq n \leq N$,
\begin{align}
\norm{\Bbar\alphahat_n - \Bbf\alphabf} &\leq C \left( \norm{\alphabf}\frac{\opnorm{\What_T - \W}}{\sigma_{r}\left(\W\right)}  + \left( \frac{\log(4N r\delta^{-1})}{n}  \bigvee \sqrt{ \frac{\log(4N r\delta^{-1})}{n} } \right)\sqrt{r} \right),
\end{align}
for some numerical constant $C>0$ that does not depend on $n,N,r,d,\delta$.

\end{theorem}

We next observe that a similar result can be obtained for matrix $\Bhat$. 
\begin{theorem}[Bias Term]\label{Th:NewBias2}
Let the RSC($cr$) condition be satisfied by $\Sigmabf$ for some large enough absolute constant $c\geq 1$. Recall that $N_0(\delta)$ is defined in \eqref{eq:N0cond1}. Then the matrix $\Bhat$ defined in \eqref{eq:Bhatdecompose} satisfies, with probability at $1-\delta$, for any $N_0(\delta)\leq n \leq N$,
\begin{align*}
\norm{\Bhat\alphahat_n - \Bbf\alphabf} &\leq C
\left( \norm{\alphabf}\frac{\opnorm{\What_T - \W}}{\sigma_{r}\left(\W\right)}  + \left( \frac{\log(4N r\delta^{-1})}{n}  \bigvee \sqrt{ \frac{\log(4N r\delta^{-1})}{n} } \right)\sqrt{r} \right),
\end{align*}
for some numerical constant $C>0$ that does not depend on $n,N,r,d,\delta$.
\end{theorem}
\paragraph{Result Discussion.} The first approach to choose an approximate representation in \eqref{eq:Bhatdecompose} not only avoids the invertibility assumption of $\Sigmabf$ but it also bypasses the closest rank $r$ approximation step in \eqref{Eq:Wbar},  thereby allowing for the design of a complete rank agnostic strategy. On the other hand the incurred constants may be quite large, so we also present the second strategy in \eqref{eq:Bdecompose}, which does require knowledge of the rank but involves smaller constants. We also point out that the perturbation analysis used to prove Theorem \ref{Thm:NewBias1} and \ref{Th:NewBias2} remains valid in the i.i.d. setting.

Then we can use the following bound on $\norm{\What_T - \W}_{F}$ given in 
\cite[Lemma 1 and Prop.~ 1]{cella2022multi}. Define for any $\delta>0$
$$
\ntrain_0(\delta): = C  \left((d+T) \bigvee \log(dTKN\delta^{-1})\right) \log(dTKN\delta^{-1}),
$$
for some numerical constant $C>0$ that does not depend on $n,N,\ntrain,r,d,K,\delta$.


We have, provided that $\ntrain\geq \ntrain_0(\delta)$, with probability at least $1-\delta$, 
\begin{align}
\label{eq:intermediaire-WhatW}
\norm{ \What_{T} - \W}_{\rm F}  =  \mathcal{O}\left( \sqrt{\frac{r}{\ntrain}} \left(  \sqrt{(T+d)} \bigvee  \sqrt{\log(\delta^{-1})}\right)\right).
\end{align}
The previous display combined with Theorems \ref{Thm:NewBias1} and \ref{Th:NewBias2} gives (up to a rescaling of the constants), for any $\delta\in (0,1)$, with probability at least $1-\delta$
\begin{align}
\label{eq:2ndterm}
    \norm{\Bbar\alphahat_n - \Bbf\alphabf} &\leq C_1\left[  \sqrt{\frac{r}{\ntrain}} \left(  \sqrt{(T+d)} \bigvee  \sqrt{\log(\delta^{-1})}\right) \bigvee \sqrt{\frac{r\log(Nr\delta^{-1})}{n}}  \right],\\
   \norm{\Bhat\alphahat_n - \Bbf\alphabf} &\leq C_2\left[  \sqrt{\frac{r}{\ntrain}} \left(  \sqrt{(T+d)} \bigvee  \sqrt{\log(\delta^{-1})}\right) \bigvee \sqrt{\frac{r\log(Nr\delta^{-1})}{n}}  \right] ,
\end{align}
with numerical constants $C_1,C_2>0$ that does not depend on $n,N,\ntrain,r,d,K,\delta$.


\subsection{Meta-Learning Linear Contextual Bandits}

\noindent We are now ready to combine the results obtained in these last two subsections. This allows us to control the meta-regret (\ref{Eq:LTLRegret}) using either $\Bbar_T$ or $\Bhat_T$ as an approximation for $\B$ in Algorithm \ref{Alg:Greedy}.

\begin{corollary}
\label{cor:main}
Let Assumptions \ref{Ass:BoundedNorms}, \ref{Ass:ArmsDistribution}, \ref{Ass:subGaussnoise} be satisfied. For any $\delta_1,\delta_2\in (0,1)$, let us assume that $\ntrain \geq \ntrain_0(\delta_1)$ with  
$$
\ntrain_{0}(\delta_1):= C  \left((d+T) \bigvee \log(dTK\ntrain\delta_1^{-1})\right) \log(dTK\ntrain\delta_1^{-1}),
$$
and
$$
\label{eq:N0cond2}
	N \geq N_0(\delta_2):= C \left(  r\bigvee \log(KN\delta_2^{-1})\right),
$$
for some large enough numerical constant $C>0$ that does not depend on $n,N,\ntrain,r,d,K,\delta_1,\delta_2$.

If $N\leq \ntrain$, then with probability at least $1-\delta_1-\delta_2$ , 
\begin{equation}
    R(T,N) \leq \bigO\left(r\sqrt{N} \left(1 \vee \sqrt{\frac{d}{T}} \vee \sqrt{\frac{\log\left(2\ntrain/\delta_1\right)}{T}} \bigvee \log(KN/\delta_2) \right)\right).
\end{equation}
\end{corollary}
\paragraph{Result Discussion.} If we consider the standard $d$-dimensional linear contextual bandit setting, the regret minimax lower-bound is of order $d\sqrt{N}$ \cite{lattimore2020bandit}(see Th. 24.2 Sec. 24.2). On the other hand, a policy, which uses the true representation $\B$, essentially works in a $r$-dimensional bandit setting and thus will incur at best the minimax-optimal regret lower bound of $r\sqrt{N}$. For $T$ large enough ($\geq d$), our regret matches this minimax-optimal lower bound, hence we can state that our policy is optimal.
\begin{proof}
We can now combine the result of Theorem \ref{Th:RegretUBound} with the ones presented in the last section. 
According to \cite{cella2022multi} (See Equation (51) in the proof of their Theorem 1), if $\lceil 
\ntrain_0(\delta_1)\rceil \leq \Nold$, then, with probability at least $1-\delta_1$, 
\begin{equation}
\label{eq:intermediaire-WhatW_MTL}
\norm{ \What_{T,\Nold} - \W_T}_{\rm F} \leq \bigO \left( \sqrt{\frac{r}{\Nold}} \left(  \sqrt{(T+d)} \bigvee  \sqrt{\log(2\Nold\delta_1^{-1})}\right) \right).
\end{equation}

Consequently, the following holds with probability at least $1-\delta_1-\delta_2$
\begin{align*}
	\sum_{n=1}^N &\gamma(T, n) \leq \sum_{n=N_0}^N Cr \Bigg[ \frac{\norm{\alphabf}}{\sigma_{r}\left(\W\right)} \sqrt{\frac{1}{\Nold\nu}} \left(1 + \sqrt{\frac{d}{T }} + \sqrt{\frac{1}{T} \log\left(\frac{2\Nold}{\delta_1}\right)} \right) + \\ 
	&\quad + \left( \frac{\log(4N r) + \log(\frac{N}{\delta_2})}{n}  \bigvee \sqrt{ \frac{\log(4N r) + \log(\frac{\ntrain}{\delta_1})}{n} } \right) \Bigg]\\
	&\leq C'r \left[  \sqrt{\frac{N^2}{\ntrain}} \left(1 + \sqrt{\frac{d}{T }} + \sqrt{\frac{1}{T} \log\left(\frac{2\Nold}{\delta_1}\right)} \right) + \left[\log(4N r) + \log\left(\frac{N}{\delta_2}\right)\right] 2 \sqrt{N} \right].
\end{align*}
We define $N_0$ according to Lemma \ref{lem:invertibility} 
as the smallest $n\in[N]$ such that the invertibility of $\Bhat^\top\widehat{\Sigmabf}_{n}\Bhat$ or $\overline{\Bbf}^\top\widehat{\Sigmabf}_{n}\overline{\Bbf}$ is guaranteed w.h.p. Consequently the estimator in \eqref{Eq:alphahat} is uniquely defined and we can apply Theorem \ref{Thm:NewBias1} or \ref{Th:NewBias2}. Specifically, according to Lemma \ref{lem:invertibility}  we have
\[
        N_0(\delta_2)= 
        C\left(r\bigvee\log(KN\delta_2^{-1})\right),
\]
where $C>0$ can depend only on $C_{\mathbf{z}}$ and $\max_{1\leq k \leq K } \left\lbrace\opnorm{\Sigma_{k}}^{1/2}\right\rbrace$.
The statement follows by combining this result with Theorem \ref{Th:RegretUBound}.
\end{proof}

\section{Conclusions}
In this paper we have investigated the problem of meta-learning with contextual linear bandit tasks. We considered the  setting in which a low dimensional representation is shared among the tasks and the goal is to use the representation inferred from multiple past tasks to facilitate learning novel tasks in the future, which share the same representation. Sharp guarantees for this problem have been recently studied in the i.i.d. statistical setting \citep{tripuraneni2021provable,boursier2022trace}, but the bandit setting has received considerably less attention. Inspired by those works, we analysed the performance gain of a bandit policy which exploits an approximate representation, estimated by a previous multitask learning procedure, when learning a new downstream task.
Specifically we showed that a simple greedy policy which uses the estimated representation satisfies a regret bound on the test task that depends on the estimation error of the representation and matches under certain conditions the bound of an optimal strategy which has full knowledge of the true representation. Our result also apply as a special case to the i.i.d. setting, offering a number of key improvements.

\bibliographystyle{plainnat}
\bibliography{reference}

\appendix

\vspace{1.0truecm}
\begin{center}
{\huge \bf Appendix} 
\end{center}

\vspace{.3truecm}
\noindent This appendix is organized in the following manner.
\begin{itemize}
    \item  Appendix \ref{sec:rankcontrol} presents an upper bound on the 
    rank of the trace-norm regularized estimator $\What_T$ in term of the rank of $\W_T$. This result provides the founding argument for the new representation estimator $\Bhat_T$
    introduced in Section~\ref{Sec:ReprLearning}.
    \item Appendix \ref{AppSec:NewBias} contains the detailed proof of Theorem \ref{Th:NewBias2} concerning the estimation error of the estimator $\Bhat_T$ on the downstream task. The proof of Theorem \ref{Thm:NewBias1} for $\Bbar_T$ is almost identical and is actually simpler as we do not need to use Lemma \ref{Le:Rank} to control the rank of the representation $\Bbar_T$. 
    
\end{itemize}


\noindent We also include the following table,
summarizing the main notation used throughout the paper.

\begin{table}[h!]
    \centering
    \renewcommand{\arraystretch}{1.2}
    \begin{tabular}{| c | c |}
        \hline Symbol & Description \\
        \hline \hline
                $[n]$ & The set $\{1,\dots,n\}$, given a  positive integer $n$\\\hline
        $T$ & Number of tasks  \\
        \hline    
        $N$ & Time horizon associated to each single task\\
        \hline    
        $d$ & Dimension of context vectors\\
        \hline
        $\W = [\w_1, \dots, \w_T] \in \reals^{d\times T}$ & Matrix of $T$ regression tasks (we also use the notation $[\W]_t \equiv w_t,~t \in [T]$)\\
        \hline            
        $r$ & Rank of the task matrix $\W$\\
        \hline
        $K$ & Number of arms\\
        \hline
$p$ & Joint distribution on $\mathbb{R}^{dK}$ (from which $K$ arm vectors are sample) 
\\ \hline
$\Dcal_{t,n}$,~~$t \in [T]$,~$n\in [N]$ & Decision sets (each containing $K$ arm vectors) sampled i.i.d. from $p$\\
\hline
        $\x_{t,n}\in\Dcal_{t,n}$ & Arm vector chosen in task $t\in[T]$ at round $n\in[N]$\\
        \hline
        $\x^*_{t,n}\in\Dcal_{t,n}$ & Optimal arm vector in task $t\in[T]$ during round $n\in[N]$\\
        \hline
        $\Sigmabf\in \reals^{d\times d}$ & Theoretical covariance matrix - see eq.~\eqref{eq:S-111}\\
        \hline
        $\Sigmabf_{t,n}\in \reals^{d\times d}$ & Adapted covariance matrix for task $t$ at round $n$
        \\
        \hline 
        $\Sigmahat_{t,n}\in \reals^{d\times d}$ & Empirical covariance matrix for task $t$ at round $n$; see eq.~\eqref{eq:S-222}\\ \hline
        $\Sigmabar\in \reals^{dT\times dT}$ & $T$-block diagonal matrix ${\rm diag}(\Sigmabf,\dots,\Sigmabf)$ \\
        \hline
        $\Sigmabar_n \in \reals^{dT\times dT}$ & $T$-block diagonal matrix ${\rm diag}(\Sigmabf_{1,n},\dots,\Sigmabf_{T,n})$ \\
        \hline
        $ \Sigmabarhat_n \in \reals^{dT\times dT}$ & $T$-block diagonal matrix ${\rm diag}(\Sigmahat_{1,n},\dots,\Sigmahat_{T,n})$\\
        \hline
        $\norm{\x},~\norm{\x}_1,~\norm{\x}_{\infty}$ & Euclidean, $\ell_1$ and maximum norm associated to a vector $\x$\\
        \hline 
        $\lambdamin(\Abf), \lambdamax(\Abf)$ & Minimum and maximum eigenvalues of a square symmetric matrix $\Abf$\\
        \hline 
        $\sigma_{\min}(\Abf), \sigma_{\max}(\Abf)$ & Minimum and maximum singular values of matrix $\Abf$\\
        \hline
        $\nucnorm{\Abf}$ & Trace norm of matrix $\Abf$ (sum of its singular values)\\
        \hline
        $\norm{\Abf}_{\rm F}$ & Frobenius norm of matrix $\Abf$ ($\ell_2$ norm o matrix elements / singular values)\\
        \hline
        $\opnorm{\Abf}$ & Operator norm of matrix $\Abf$ (maximum singular value)\\
        \hline
        $\Im(\Abf)$ & The range of matrix $\Abf$\\
        \hline
    \end{tabular}
    \label{Tab:Notation}
\end{table}


\section{Controlling the Rank of $\What_T$}
\label{sec:rankcontrol}

Define the event
\begin{equation}\label{Eq:NoisyBound}
    \Omega_{\ntrain}= \left\{ \lambda_{\ntrain} \geq \frac{4}{\ntrain} \norm{\sum_{t=1}^T\sum_{n=1}^{\ntrain}\eta_{t,n} \x_{t,n} \otimes \e_t}\right\}.
    \end{equation}

\begin{lemma}[Controlling the Rank]\label{Le:Rank}
Let the RSC($r$) condition be satisfied by $\Sigmabf$. Then there exists an absolute constant $C>0$ such that any solution of (\ref{Eq:What}) satisfies on the event $\Omega_{\ntrain}$
$$
\rho(\What_T)\leq C\frac{r}{\kappa(\Sigmabarhat_{\ntrain})},
$$
where $\Sigmabarhat_{\ntrain}= {\rm diag}(\Sigmahat_{1,\ntrain},\dots,\Sigmahat_{T,\ntrain})\in \reals^{dT\times dT}$.
\end{lemma}

\begin{proof}

We define first the SVD-decomposition of $\What_T$ as
\[
   \What_T = \sum_{j=1}^{\rhat} \sigma_j(\What_T) \uhat_j(\What_T) \otimes \vhat_j(\What_T) 
\]
where we set $\rhat = \rho(\What_T)$ and $\sigma_1(\What_T)\geq \cdots\geq \sigma_{\rhat}(\What_T)>0$ are the singular values of $\What$ with corresponding left and right singular vectors $\uhat_j(\What_T)$ and $\vhat_j(\What_T)$ respectively. We also define the support spaces of $\What_T$ as 
$$
\Shat_1:=\mathrm{l.s.}(\uhat_1(\What_T),\ldots,\uhat_{\rhat}(\What_T)),\quad \Shat_2:=\mathrm{l.s.}(\vhat_1(\What_T),\ldots,\vhat_{\rhat}(\What_T)).
$$

As it was defined in Equation (\ref{Eq:What}), we have
\[
    \What_T \in \argmin_{\W \in \reals^{d\times T}}\left\{ \frac{1}{\ntrain} \sum_{t,n=1}^{T,\ntrain} \left(y_{t,n} - \langle \x_{t,n}, \w_{t,n} \rangle \right)^2+ \lambda_{\ntrain} \norm{\W}_* \right\}.
\]
The KKT condition gives, for every $\Hbf \in\reals^{d\times T}$, with $\Hbf=(\h_1|\dots|\h_T)$,
\begin{equation*}
    -\frac{2}{\ntrain}\sum_{t=1}^T\sum_{n=1}^{\ntrain}\left(y_{t,n} - \langle \x_{t,n}, \what_{t} \rangle \right) \langle \x_{t,n}, \h_t\rangle + \lambda_{\ntrain} \langle \Vhat, \Hbf \rangle = 0
\end{equation*}
where 
\begin{equation*}
    \Vhat\in\partial\norm{\What_T}_* = \left\{ \sum_{j=1}^{\rhat} \uhat_j \otimes \vhat_j + \Pbf_{\Shat^\bot_1} \A \Pbf_{\Shat^\bot_2}, \A \in \reals^{d\times T}: \norm{\A}\leq 1\right\}.
\end{equation*}

For $\Hbf^j = \uhat_j \otimes \vhat_j$ with $1\leq j\leq \rhat$, the following holds
\[
    0 = -\frac{2}{\ntrain} \sum_{t=1}^T \sum_{n=1}^{\ntrain} \langle \x_{t,n}, \w_t - \what_t\rangle \langle \x_{t,n}, \Hbf^j_t\rangle - \frac{2}{\ntrain}\sum_{t=1}^T\sum_{n=1}^{\ntrain} \eta_{t,n} \langle \x_{t,n}, \Hbf^j_t \rangle + \lambda_{\ntrain} \langle \Vhat, \Hbf^j_t \rangle.
\]
On the event \eqref{Eq:NoisyBound}, the following holds,
\begin{equation*}
     - \frac{2}{\ntrain}\sum_{t=1}^T\sum_{n=1}^{\ntrain} \eta_{t,n} \langle \x_{t,n}, \Hbf^j_t \rangle = - \frac{2}{\ntrain} \langle \sum_{t=1}^T\sum_{n=1}^{\ntrain} \eta_{t,n} \x_{t,n} \otimes \e_t, \Hbf^j \rangle_{HS} = \langle \Acal^*\left(\eta_{t,n}\right), \Hbf^j\rangle_{HS}.
\end{equation*}
Thus,
\begin{equation*}
    \left \lvert \frac{2}{\ntrain}\sum_{t=1}^T \sum_{n=1}^{\ntrain}\eta_{t,n} \langle \x_{t,n}, \Hbf^j_t \rangle \right\rvert \leq 2 \norm{\sum_{t=1}^T \eta_{t,n} \x_{t,n} \otimes \e_t} \norm{\Hbf^j}_*.
\end{equation*}
Define $\Acal:\reals^{d\times T}\rightarrow\reals^{n\times T}$ as $\left[\Acal[\Hbf^j]\right]_n^t = \langle \x_{t,n} \otimes \e_t, \Hbf^j\rangle_{HS}$, then we have
\begin{equation*}
    \frac{1}{\ntrain} \langle \Acal(\Delta_{\ntrain}) , \Hbf^j \rangle_{HS} = \frac{1}{\ntrain} \sum_{t=1}^T\sum_{n=1}^{\ntrain} \left(\Hbf_t^j\right)^\top \left(\x_{t,n} \otimes \x_{t,n}\right) \left(\W-\What\right)_t.
\end{equation*}
On the event $\Omega_{\ntrain}$ we have
\[
    \frac{2}{\ntrain}\sum_{t=1}^T\sum_{n=1}^{\ntrain} \left(\Hbf_t^j\right)^\top \left(\x_{t,n}\otimes\x_{t,n}\right)\left(\W-\What\right)_t \geq \frac{\lambda_{\ntrain}}{2}
 \hspace{1em} \forall j \in 1,\dots,\rhat.\]
Thus we have
\begin{eqnarray*}
    \frac{1}{{\ntrain}^2}\sum_{t=1}^T \norm{\left[\Acal(\Delta_{\ntrain})\right]^j}_2^2 & = & \frac{1}{{\ntrain}^2} \norm{\Acal(\Delta_{\ntrain})}_F^2 \\
    & \geq & \frac{1}{N^2}\sum_{j=1}^{\rhat} \langle \Acal(\Delta_{\ntrain}), \Hbf^j\rangle_{HS}^2 \geq \sum_{j=1}^{\rhat} \frac{\lambda^2_{\ntrain}}{16} = \rhat \frac{\lambda_{\ntrain}^2}{16}.
\end{eqnarray*}
Next, always on $\Omega_{\ntrain}$, we have $\frac{1}{{\ntrain}^2}\norm{\Acal(\Delta_{\ntrain})}^2_F \leq \frac{2}{\ntrain}\lambda_{\ntrain}\norm{\Delta_{\ntrain}}_*$ thanks for \citep{cella2022multi}[Lemma 1]. Additionally, as $\norm{\Delta_{\ntrain}}_*\leq 4 \norm{\Delta'_{\ntrain}}_*$ and since $\Delta'_{\ntrain}$ is of rank at most $2r$ we get 
\[
    \frac{1}{{\ntrain}^2}\norm{\Acal(\Delta_{\ntrain})}^2_F \leq \frac{8}{\ntrain}\lambda_{\ntrain}\norm{\Delta'_{\ntrain}}_*.
\]
Combining the last two displays, we get on event $\Omega_{\ntrain}$
\[
    \frac{\rhat}{16}\lambda_{\ntrain}^2 \leq \frac{8}{\ntrain} \lambda_{\ntrain} \norm{\Delta'_{\ntrain}}_* \leq \frac{8\sqrt{2}}{\ntrain} \lambda_{\ntrain} \sqrt{r} \norm{\Delta_{\ntrain}}_F.
\]
Since $\norm{\Delta_{\ntrain}}_F\leq \frac{32\lambda_{\ntrain}\sqrt{r}}{\kappa(\Sigmabarhat_{\ntrain})}$ we obtain $\rhat\leq C\frac{r}{\kappa(\Sigmabarhat_{\ntrain})}$ for some absolute constant $C>0$. 
\end{proof}

Define
\begin{align}
    \label{eq:N0sampleT}
  \ntrain_0(x):=  C'\left(\frac{\nu \omega_\Xcal \max_{1\leq k\leq K}\left\lbrace \opnorm{\Sigma_k} \right\rbrace}{\kappa(\Sigmabar)} \bigvee 1 \right) \frac{\nu \omega_\Xcal C_{TKN}(x)}{\kappa(\Sigmabar)}(x+\log(dT)),
\end{align}
for some large enough absolute constant $C'>0$ and 
$$
C_{TKN}(x) = C \left( d+ \sqrt{d(x+\log(TK\ntrain))}\vee (x+\log(TK\ntrain))\right)$$ 
with
$$
C=C(\max_{1\leq k\leq K}\left\lbrace \|\Sigmabf_k\|_{\mathrm{op}}\right\rbrace,C_{\mathbf{z}}))>0.
$$
See \cite{cella2022multi}[Lemma 2] for more details.

Under the RSC($r$) condition satisfied by $\Sigmabf$ and Lemma 4 in \cite{cella2022multi}, we get for any $\ntrain\geq \ntrain_0(x)$, with probability at least $1-e^{-x}$, that $\kappa(\Sigmabarhat_{\ntrain})\geq \frac{\kappa\big(\Sigmabf\big)}{4 \nu \omega_\Xcal}>0.$. According to Proposition 1 in  \cite{cella2022multi}, we also have with that $\mathbb{P}\left(\Omega_{\ntrain} \right)\geq 1-e^{-x}$, provided that
\begin{align}
\label{eq:stoproofthm1}
\lambda_{\ntrain} = \lambda_{\ntrain}(x)&:= C \left(  \frac{T+d}{\ntrain} \bigvee \frac{x+\log(2N)}{\ntrain} \bigvee \sqrt{\frac{(T+d)}{\ntrain}} \bigvee  \sqrt{ \frac{x+ \log(2N)}{\ntrain}}\right)
\end{align}
for some large enough constant $C = C\left(c_{\eta}, C_{\mathbf{z}},\sigma,\max_{k}\left\lbrace\|\Sigma_{k}^{1/2}\|_{\mathrm{op}}\right\rbrace\right)>0$. A standard uniond bound (up to a rescaling of the constants) gives in combination with Lemma \ref{Le:Rank}, with probability at least $1-e^{-x}$,
\begin{align}
\label{eq:rankWhat}
\rho(\What_T)\lesssim C r,
\end{align}
for some numerical constant $C = C\left(c_{\eta}, C_{\mathbf{z}},\sigma,\nu ,\omega_\Xcal,\kappa\big(\Sigmabar\big),\max_{k}\left\lbrace\|\Sigma_{k}^{1/2}\|_{\mathrm{op}}\right\rbrace\right)>0$.

\section{Proof of Theorem \ref{Th:NewBias2}
} \label{AppSec:NewBias}

We present the proof for the bandit martingale setting which also includes the i.i.d. setting. Indeed the stochastic objects introduced in the proofs of Lemmas 3 through 7 satisfy a martingale structure which is still valid when we replace our $K$ arms bandit policy by i.i.d. observations with $K=1$. Lemma 2 does not require any stochastic structure. In order to obtain the control of the rank with high probability, this lemma is combined next with Proposition 1 in \cite{cella2022multi} which is also valid in the i.i.d.

For any matrix $\A$ of rank at least $r$, we denote by $\sigma_{r}\left(\A\right)$ its $r$-th largest singular value.

\begin{proof}[Proof of Theorem \ref{Th:NewBias2}]
We prove the result for $\Bhat$. The proof for $\Bbar$ is almost identical and actually simpler.
Recall that $\Bbf\in\reals^{d\times r}$ be a matrix of rank $r$ with orthonormal columns. We denote by $\V$ the range $\Im(\Bbf)$. Analogously, we denote by $\Vhat$ the range $\Im(\Bhat)$ of $\Bhat\in \reals^{d\times \tilde{r}}$. We set $\tilde{r}:=\rho(\Bhat)$. According to \eqref{eq:rankWhat}, we have with probability at least $1-e^{-x}$ that  $\tilde{r}\leq C r$ for some constant $C>0$ provided that $\ntrain\geq \ntrain_0(x)$. We assume that $c$ is large enough such that $C<c-1$. 



We denote by $\lambda_{\tilde{r}}(\Sigmabf)$ the $\tilde{r}$-th largest eigenvalue of $\Sigmabf$. We also recall that $\widehat{\Sigmabf}_{n}:=\frac{1}{n}\X_{n}\X_{n}^\top$ is the empirical version of $\Sigmabf$. Since $\Bhat$ is built based on the recorded history of $T$ prior tasks, it can be considered frozen when we consider the new $(T+1)$-th task run on $N$ rounds. Thus we can apply Lemma \ref{lem:invertibility} with $\overline{\mathbf{M}} = \Bhat$. We get, for any $n\geq N_0(x)$
,with probability at least $1-e^{-x}$,
\begin{align}
\label{eq:proofrankbias1}
    \lambda_{\tilde{r}}(\Bhat^\top\widehat{\Sigmabf}_{n}\Bhat)>\frac{\kappa(\Sigmabf)}{4\nu \omega_{\Xcal}}>0.
\end{align}


At any round $n\geq N_0(x)$, considering $\Y_{n} = [y_{T+1,1},\dots,y_{T+1,n}]^\top\in\reals^n$, $\etab_{n} = [\eta_{T+1,1},\dots,\eta_{T+1,n}]\in\reals^{n}$ and $\X_{n} = \left[\x_{T+1,1}|\dots|\x_{T+1,n}\right] \in\reals^{d\times n}$ we have
\[
    \Y_{n} = \X_{n}^\top \B \alphabf + \etab_n \in \reals^{n}.
\]
In agreement with \eqref{Eq:alphahat}, since $\alphahat =\alphahat_{T+1,n}$ is a solution to the minimization problem
$\min_{\ubf\in\reals^r}\norm{\Y_{n} - \X_{n}^\top \Bhat \ubf}^2$. In view of \eqref{eq:proofrankbias1}, $\alphahat_{n}$ is unique and
\[
    \alphahat_{n} = \left(\Bhat^\top\X_{n}\X_{n}^\top\Bhat\right)^{-1} \Bhat^\top \X_{n}\Y_{n} = \frac{1}{n} (\Bhat^\top\widehat{\Sigmabf}_{n}\Bhat)^{-1}\X_{n} \Y_{n},
\]
where $\widehat{\Sigmabf}_{n}:=\frac{1}{n}\X_{n}\X_{n}^\top$ is the empirical version of $\Sigmabf$. 


We recall that $\Bhat\in \reals^{d\times \tilde{r}}$ and $\alphahat_{n}\in \reals^{\tilde{r}}$ whereas $ \Bbf\in \reals^{d\times r}$ and $\alphabf\in \reals^{r}$ with $r\leq \tilde{r}$. To avoid any incompatibility of dimension in the following analysis, we complete both $\Bbf$ and $\Bhat$ into $d\times d$ orthogonal projection matrices $\Pbf= ( \Bbf\vert \mathbf{0}_{d\times (d-r)})$ and $\Phat = ( \Bhat\vert \mathbf{0}_{d\times (d-\tilde{r})})$. Similarly we augment vectors $\alphabf\in \reals^{r}$ and $\alphahat_n\in \reals^{\tilde r}$ into $d$ dimensional vectors as follows $\overline{\alphabf} = (\alphabf^\top \vert \mathbf{0}_{d-r}^\top)^\top$ and $\overline{\alphahat} = (\alphahat^\top \vert \mathbf{0}^\top_{d-\tilde{r}})^\top$. 
Thus we have $\B\alphabf= \Pbf\overline{\alphabf}$, $\Bhat\alphahat_{n} =\Phat\overline{\alphahat}_{n} $ and $\Bhat(\alphabf^\top\vert \mathbf{0}_{\tilde{r}-r}^\top)^\top= (\B\vert \mathbf{0}_{d\times (\tilde{r}-r}))(\alphabf^\top\vert \mathbf{0}_{\tilde{r}-r}^\top)^\top = \Phat\overline{\alphabf}$ .

In view of \eqref{eq:proofrankbias1}, we have $\left(\Bhat^\top\X_{n}\X_{n}^\top \Bhat\right)^{-1}\left(\Bhat^\top\X_{n}\X_{n}^\top \Bhat\right) = \I_r$. Thus we get
\begin{align*}
    \Bhat\alphahat_{n} &= \Bhat\left(\Bhat^\top\X_{n}\X^\top_{n}\Bhat\right)^{-1}\Bhat^\top \X_{n}\left(\X_{n}^\top \Bbf \alphabf  + \etab_{n}\right)\\
    &=\Phat\overline{\alphabf} + \left(\Bhat^\top\X_{n}\X_{n}^\top \Bhat\right)^{-1}\left[ \Bhat^\top\X_{n}\X_{n}^\top \left(\Pbf - \Phat \right) \overline{\alphabf} + \Bhat^\top \X_n \etab_n \right].
\end{align*}

Thus we get the following decomposition
\begin{align}
\label{eq:intermbiasterm1}
    \Bhat\alphahat_{n} - \B\alphabf = (\Phat-\Pbf)\overline{\alphabf}  + \left(\Bhat^\top\X_{n}\X_{n}^\top \Bhat\right)^{-1}\left[ \Bhat^\top\X_{n}\X_{n}^\top \left(\Pbf - \Phat \right) \overline{\alphabf} + \Bhat^\top \X_n \etab_n \right].
\end{align}

Denote by $\overline{\Pbf}$ the orthogonal projection onto $\Im\left( \Phat\right)+\Im\left( \Pbf\right)$. Recall \eqref{eq:proofrankbias1}. Thus the second term can be upper bounded by
\begin{align}
\label{eq:intermbiasterm2}
   \norm{\Bhat\left(\Bhat^\top\widehat{\Sigmabf}_{n}\Bhat\right)^{-1} \Bhat^\top \widehat{\Sigmabf}_{n} \left(\Pbf - \Phat \right) \bar{\alphabf}}
   &\leq \opnorm{\Bhat}\opnorm{\left(\Bhat^\top\widehat{\Sigmabf}_{n}\Bhat\right)^{-1} } \opnorm{\Bhat^\top \widehat{\Sigmabf}_{n} \overline{\Pbf}}\norm{(\Phat - \Pbf )\overline{\alphabf}}\notag\\
   &\leq \frac{1}{\lambda_{\tilde{r}}(\Bhat^\top\widehat{\Sigmabf}_{n}\Bhat)} \opnorm{\Bhat^\top \widehat{\Sigmabf}_{n} \overline{\Pbf}} \norm{(\Phat - \Pbf )\overline{\alphabf}}\notag\\
   &\leq\frac{4\nu \omega_{\Xcal}}{\kappa(\Sigmabf)} \opnorm{\Bhat^\top \widehat{\Sigmabf}_{n} \overline{\Pbf}} \norm{(\Phat - \Pbf )\overline{\alphabf}},
\end{align}
since we obviously have $\opnorm{\Bhat} =1$.

We note that $\bar{r} = \rho(\overline{\Pbf})\leq C r$ on an event of probability greater than $1-e^{-x}$ in view of \eqref{eq:rankWhat}. Thus, there exists a $d\times \bar{r}$ matrix $\overline{\mathbf{M}}$ with orthonormal columns such that  $\opnorm{\Bhat^\top \widehat{\Sigmabf}_{n} \overline{\Pbf}} = \opnorm{\Bhat^\top \widehat{\Sigmabf}_{n} \overline{\mathbf{M}}}$. Next, lemma \ref{Le:QuadraticForm} gives with probability at least $1-e^{-x}$, for any $n\in [N]$, that
\[
    \opnorm{\Bhat^\top \widehat{\Sigmabf}_{n} \overline{\mathbf{M}}- \Bhat^\top \Sigmabf_{n} \overline{\mathbf{M}} } \lesssim_{C_{\x}}\left(\sqrt{\frac{r}{n}} \bigvee \sqrt{\frac{x+ \log(N)}{n}} \bigvee \frac{r}{n} \bigvee \frac{x+ \log(N)}{n} \right),
\]
where we recall that 
$$
\Sigmabf_{n} := \frac{1}{n}\sum_{s=1}^n \mathbb{E}\left[ \x_{T+1,s}\x_{T+1,s}^\top \bigg\vert \Fbar_{s-1}\right].
$$
Consequently, with the same probability, we have for any $n\in [N]$
\begin{align*}
        \opnorm{\Bhat^\top \widehat{\Sigmabf}_{n} \overline{\mathbf{M}}} &\leq \opnorm{\Bhat^\top \Sigmabf_{n} \overline{\mathbf{M}} }+ \opnorm{\Bhat^\top (\widehat{\Sigmabf}_{n}-\Sigmabf_{n}) \overline{\mathbf{M}} }\\
    &\leq \opnorm{\Bhat^\top \Sigmabf_{n} \overline{\mathbf{M}} }+ C\left(\sqrt{\frac{r}{n}} \bigvee \sqrt{\frac{x+ \log(N)}{n}} \bigvee \frac{r}{n} \bigvee \frac{x+ \log(N)}{n} \right),
\end{align*}
for some numerical constant $C>0$ that can depend only on $C_{\x}$.



Lemma \ref{lem:controlsto} gives 
\begin{align*}
      \opnorm{\Bhat^\top \Sigmabf_{n} \overline{\mathbf{M}} }&\leq 
      \max_{1\leq k \leq K } \left\lbrace\opnorm{\Sigma_{k}}\right\rbrace.
\end{align*}



Combining the last two displays gives for any $x>0$, with probability at least $1-e^{-x}$, for any $n\in [N]$, that
\[
    \opnorm{\Bhat^\top \widehat{\Sigmabf}_{n} \overline{\mathbf{M}}} \leq  \max_{k\in [K]}\lbrace  \opnorm{ \Sigmabf_k }\rbrace  + C
   \left(\sqrt{\frac{r}{n}} \bigvee \sqrt{\frac{x+ \log(N)}{n}} \bigvee \frac{r}{n} \bigvee \frac{x+ \log(N)}{n} \right).
\]
Note that when $n\geq N_0(x)$, we have with probability at least $1-e^{-x}$
\[
    \opnorm{\Bhat^\top \widehat{\Sigmabf}_{n} \overline{\mathbf{M}}} \lesssim  2 \max_{k\in [K]}\lbrace  \opnorm{ \Sigmabf_k }\rbrace.
\]

Combining the previous display with \eqref{eq:intermbiasterm2} (up to a rescaling of the constants) gives with probability at least $1-e^{-x}$
\begin{align}
\hspace{-.2truecm}\label{eq:intermbiasterm3}
   \norm{\Bhat\left(\Bhat^\top\widehat{\Sigmabf}_{n}\Bhat\right)^{-1} \Bhat^\top \widehat{\Sigmabf}_{n} \left(\Pbf - \Phat \right) \bar{\alphabf}}
   &\leq\frac{4\nu \omega_{\Xcal}\, (1\vee \max_{k\in [K]}\lbrace  \opnorm{ \Sigmabf_k }\rbrace)}{\kappa(\Sigmabf)}  \norm{(\Phat - \Pbf )\overline{\alphabf}}.
\end{align}

The $d\times T$ matrix $\What_{T}$ admits SVD
$$
\What_{T} = \sum_{j=1}^{\tilde{r}} \sigma_j(\What_{T}) \uhat_j\otimes \vhat_j,
$$
with singular values $ \sigma_1(\What_{T})\geq \cdots\geq  \sigma_{\tilde{r}}(\What_{T})$ and corresponding left and right singular vectors $(\uhat_j)_{j}$, $(\vhat_j)_{j}$. Note that
$$
\Phat = \sum_{j=1}^{\tilde{r}} \uhat_j\otimes\uhat_j.
$$
We set $\Phat_{\leq r} =\sum_{j=1}^{r} \uhat_j\otimes\uhat_j $ and $\Phat_{> r} =\sum_{j=r+1}^{\tilde{r}} \uhat_j\otimes\uhat_j $. We have
\begin{align*}
  \norm{(\Phat - \Pbf)\overline{\alphabf}}  &= \norm{(\Phat_{\leq r} - \Pbf)\overline{\alphabf}+ \Phat_{> r} \overline{\alphabf}} \\
  &\leq \norm{(\Phat_{\leq r} - \Pbf)\overline{\alphabf}}+
  \norm{\Phat_{> r} \overline{\alphabf}}\\
  &\leq \norm{(\Phat_{\leq r} - \Pbf)\overline{\alphabf}}+
  \norm{\Phat_{> r}(\Pbf_{\leq r}(\overline{\alphabf}))},
\end{align*}
where we have used in the last line that $\overline{\alphabf} =\Pbf_{\leq r}(\overline{\alphabf})$ by definition of $\overline{\alphabf}$ and $\Pbf =\Pbf_{\leq r}$.
Next we note that
\begin{align*}
    \Phat_{> r}\Pbf_{\leq r} &=  (\Phat_{> r} - \Pbf_{> r})\Pbf_{\leq r} = \biggl( (\I- \Phat_{\leq r}) - (\I - \Pbf_{\leq r}) \biggr) = \Pbf_{\leq r} - \Phat_{\leq r} = \Pbf - \Phat_{\leq r} .
\end{align*}
Combining the last two displays gives
\begin{align}
\norm{(\Phat - \Pbf)\overline{\alphabf}}  \leq 2
\norm{(\Pbf - \Phat_{\leq r})(\overline{\alphabf})}\leq \opnorm{\Phat_{\leq r} - \Pbf}\norm{\overline{\alphabf}}.
\end{align}

In order to get a bound on $\opnorm{\Phat_{\leq r} - \Pbf}$, we will use the following standard perturbation bound (See for instance Theorem 3 in \cite{blanchard2005})
\[
\opnorm{\Phat_{\leq r} - \Pbf} \leq  \frac{1}{\sigma_{r}\left(\W\right)} \opnorm{\What_T - \W},
\]
where $\sigma_{r}\left(\W\right)$ is the $r$-th largest singular value of $\W$. Note this bound in \cite{blanchard2005} is formulated for symmetric matrices but this result can be immediately extended to arbitrary rectangular matrices by the well-known symmetrization trick. See for instance \citep{KoltchinskiiXia} page 398 for more details on this argument.

Combining the last two displays with \eqref{eq:intermbiasterm3}, we get (up to a rescaling of the constants), with probability at least $1-e^{-x}$, for any $n\geq N_0(x)$, that
\begin{align}
\label{eq:intermbiasterm4}
\hspace{-.2truecm}   \norm{\Bhat\left(\Bhat^\top\widehat{\Sigmabf}_{n}\Bhat\right)^{-1}\hspace{-.1truecm}   \Bhat^\top \widehat{\Sigmabf}_{n} \left(\bar{\Bbf} {-} \Bhat \right) \bar{\alphabf}}
   &\leq \frac{4\nu \omega_{\Xcal}\, (1\vee \max_{k\in [K]}  \opnorm{ \Sigmabf_k })}{\kappa(\Sigmabf)} \norm{\overline{\alphabf}}\frac{\opnorm{\What_T {-} \W}}{\lambda_{r}\left(\W\right)}.
\end{align}


We handle now the third term in \eqref{eq:intermbiasterm1}. We have
$$
\left(\Bhat^\top\X_{n}\X_{n}^\top \Bhat\right)^{-1}\Bhat^\top \X_n \etab_n= \frac{1}{n}\left(\Bhat^\top\widehat{\Sigmabf}_{n}\Bhat\right)^{-1} \Bhat^\top \X_n \etab_n.
$$
Hence, in view of \eqref{eq:proofrankbias1},
\begin{align*}
    \norm{\frac{1}{n}\left(\Bhat^\top\widehat{\Sigmabf}_{n}\Bhat\right)^{-1} \Bhat^\top \X_n \etab_n} &\leq \opnorm{\left(\Bhat^\top\widehat{\Sigmabf}_{n}\Bhat\right)^{-1}}\norm{\Bhat^\top \X_n \etab_n}\leq \frac{1}{n} \frac{4\nu \omega_{\Xcal}}{\kappa(\Sigmabf)}\frac{1}{n}\norm{\Bhat^\top \X_n \etab_n}.
\end{align*}
We have 
\begin{align*}
    \norm{\Bhat^\top \X_n \etab_n}^2 &= \sum_{j=1}^{\tilde{r}} \langle\vhat_j, \X_n\etab_n \rangle^2 = \sum_{j=1}^{\tilde{r}} \langle \X_n^\top \vhat_j, \etab_n \rangle^2 = \sum_{j=1}^{\tilde{r}} \left( \sum_{s=1}^{n} \eta_s \langle \x_s , \vhat_j\rangle \right)^2.
\end{align*}

Lemma \ref{lem:technorm} combined with an union bound guarantees that, for any $x>0$, with probability at least $1-e^{-x}$, for any $n\geq N_0(x)$, that
\begin{align}
\label{eq:3rdterm}
    \frac{1}{n}\norm{\Bhat^\top \X_n \etab_n} & \lesssim \left(\sqrt{2} c_{\eta} C_{\x} \, \frac{\log(4N\tilde{r}) + x)}{n}  + \sigma\, \max_{k \in [K]} \opnorm{\Sigma_{k}}^{1/2}\ \,\sqrt{2\, \frac{\log(4N\tilde{r}) + x)}{n} }\right)\sqrt{r}.
\end{align}

Combining \eqref{eq:intermbiasterm1}, \eqref{eq:intermbiasterm4}, and \eqref{eq:3rdterm} (up to a rescaling of the constants) gives the result.
\end{proof}

\begin{lemma}
\label{lem:controlsto}
Let Assumption \ref{Ass:BoundedNorms} be satisfied. We have
\begin{align}
      \opnorm{\Bhat^\top \Sigmabf_{n} \overline{\mathbf{M}} }&\leq \max_{k \in [K]} \opnorm{\Sigma_{k}}.
\end{align}

\end{lemma}

\begin{proof}[Proof of Lemma \ref{lem:controlsto}]

By definition of $\Sigmabf_{n}$ and the operator norm, we have
\begin{align*}
    \opnorm{\Bhat^\top \Sigmabf_{n} \overline{\mathbf{M}} } &= \max_{\ubf \in \Scal^{\tilde{r}-1},\, \vbf \in \Scal^{\bar{r}-1}}\left\lbrace  \frac{1}{n}\sum_{s=1}^n \mathbb{E}\left[ \langle \ubf, \Bhat^\top \x_{T+1,s}\rangle \langle \overline{\mathbf{M}}^{\top}\x_{T+1,s},\vbf \rangle   \bigg\vert \Fbar_{s-1}\right] \right\rbrace\\
    &= \max_{\ubf \in \Scal^{\tilde{r}-1},\, \vbf \in \Scal^{\bar{r}-1}}\left\lbrace  \frac{1}{n}\sum_{s=1}^n  \langle \ubf,\Bhat^\top \Sigmabf_{k(s)}\overline{\mathbf{M}}\vbf \rangle   \right\rbrace,
\end{align*}
where at any round $s\in [n]$, $\x_{T+1,s}$ admits covariance $\Sigmabf_{k(s)}$ for some $k(s)\in [K]$.

Consequently, 
\begin{align*}
    \opnorm{\Bhat^\top \Sigmabf_{n} \overline{\Bbf} } &\leq  \max_{k \in [K]} \opnorm{\Sigma_{k}}.
\end{align*}

\end{proof}

\begin{lemma}
\label{lem:invertibility}
Let Assumptions \ref{Ass:BoundedNorms} and \ref{Ass:ArmsDistribution} be satisfied. Let $\Sigmabf$ satisfies the RSC($\bar{r}$) condition for some integer $\bar{r}\geq 1$ with constant $\kappa(\Sigmabf)$. Let $\overline{\Bbf} \in \reals^{d\times \bar{r}}$ admits orthonormal column vectors. Then, for any $n\geq N_0(x)$ with
$$
N_0(x):= C_{C_{\mathbf{z}}, \nu \omega_{\Xcal},\kappa(\Sigmabf),\max_{k\in [K]}\opnorm{\Sigmabf_k}} \left( \bar{r}(x+\log(2KN))\bigvee (x+\log(2KN))^2\right),
$$
we have with probability at least $1-e^{-x}$
$$
\lambda_{\bar{r}}(\overline{\Bbf}^\top\widehat{\Sigmabf}_{n}\overline{\Bbf})>\frac{\kappa(\Sigmabf)}{4\nu \omega_{\Xcal}}>0.
$$
\end{lemma}

\begin{proof}[Proof of Lemma \ref{lem:invertibility}]
Recall that
$$
\widehat{\Sigmabf}_{n} :=\frac{1}{n}\sum_{s=1}^n  \x_{T+1,s}\x_{T+1,s}^\top,\quad \Sigmabf_{n} := \frac{1}{n}\sum_{s=1}^n \mathbb{E}\left[ \x_{T+1,s}\x_{T+1,s}^\top \bigg\vert \Fbar_{s-1}\right].
$$

We set $\overline{\V} = \Im(\overline{\Bbf})$. We have
\begin{align*}
    \opnorm{\overline{\Bbf}^\top\Sigmahat_n\overline{\Bbf} - \overline{\Bbf}^\top\Sigmabf_{n}\overline{\Bbf}} &= \max_{v\in \overline{\V}}\left\lbrace  \frac{1}{n}\sum_{s=1}^{n}\langle \x_{T+1,s},v \rangle^2 - \mathbb{E}\left[\langle \x_{T+1,s},v\rangle^2 \vert \Fcal_{s-1} \right] \right\rbrace.
\end{align*}
We apply Lemma 2 in \cite{cella2022multi} with  $T=1$ and $\Sigmabarhat_n$ and $d$ replaced by $\Bbar^\top\widehat{\Sigmabf}_{n}\Bbar$ and $\bar{r}$ respectively.
In view of Assumption \ref{Ass:BoundedNorms}, we get for any $x>0$, with probability at least $1-e^{-x}$, for any $n\in [N]$,
\begin{align}
    &\opnorm{\Bbar^\top\Sigmahat_n\Bbar - \Bbar^\top\Sigmabf_{n}\Bbar} \leq C \max_{k \in [K]}  \opnorm{\Bbar^\top\Sigmabf_k\Bbar} \left[ \left( \sqrt{ \bar{r}} \bigvee  \sqrt{x+\log(2KN)} \right) \sqrt{\frac{x+ \log(2\bar{r})}{n}}\right.\notag\\ 
    &\hspace{6.5cm}\left. \bigvee \biggl(\bar{r} \bigvee  (x+\log(2KN)\biggr)\frac{x+ \log(2\bar{r})}{n}\right]=:\overline{\lambda}(x),
\end{align}
where the constant $C>0$ can depend only $C_{\mathbf{z}}$.

We deduce from the previous display for any $x>0$, with probability at least $1-e^{-x}$
$$
\lambda_{\bar{r}}(\Bbar^\top\Sigmahat_n\Bbar) \geq \lambda_{\bar{r}}(\Bbar^\top\Sigmabf_{n}\Bbar) - \overline{\lambda}(x).
$$
By definition of the RSC($\bar{r}$) condition, we immediately get $\lambda_{\bar{r}}(\Bbar^\top\Sigmabf_n\Bbar)\geq \kappa(\Sigmabf_n)$. Next, relying on \citep[Lemma 10]{oh2020sparsity}, under Assumption \ref{Ass:ArmsDistribution} we have
\begin{equation}\label{Eq:UBProofLemma2}
    \Sigmabf_{n} = \frac{1}{n}\sum_{s=1}^n \E\left[ \x_{T+1,s}  \x_{T+1,s}^\top  \Big| \Fbar_{s-1} \right]  \succeq {\left(2 \nu \omega_\Xcal\right)}^{-1} \Sigmabf. 
\end{equation}
Hence if $\Sigmabf$ satisfies the RSC($\bar{r}$) condition with constant $\kappa(\Sigmabf)$, then $\Sigmabf_n$ satisfies the RSC($\bar{r}$) condition with constant $\frac{\kappa(\Sigmabf)}{2\nu \omega_{\Xcal}}$. This fact combined with the last two displays gives, provided that $\overline{\lambda}(x)\leq \frac{\kappa(\Sigmabf)}{4\nu \omega_{\Xcal}}$, with probability at least $1-e^{-x}$
$$
\lambda_{\bar{r}}(\Bbar^\top\Sigmahat_n\Bbar) \geq \frac{\kappa(\Sigmabf)}{4\nu \omega_{\Xcal}}.
$$

\end{proof}

Recall that $\Bhat$ is built based on the recorded history of $T$ prior tasks. Thus, it can be considered frozen when we consider the new $(T+1)$-th task run on $N$ rounds and we can readily apply Lemma \ref{lem:invertibility} with $\Bbar = \Bhat$.

We will need the following result. 

\begin{lemma}
\label{lem:norm-arms}
Let Assumption \ref{Ass:BoundedNorms} be satisfied. Consider an orthogonal projection $P_U$ of rank $\bar{r}$. Then, for any $x>0$, with probability at least $1-e^{-x}$,
\begin{align}
    \max_{1\leq n\leq N}\max_{1\leq t\leq T}\max_{\x_{t,n}\in \mathcal{D}_{t,n}} \|P_U(\x_{t,n})\|_2^2 \leq C_{T,K,N,\bar{r}}(x), 
\end{align}
where
\begin{align*}
& C_{T,K,N,\bar{r}}(x) := \max_{k \in [K]} \mathrm{trace}(P_U\Sigma_{k}P_U)\\
&\hspace{0.5cm}+ C_{\mathbf{z}}\left( \max_{k \in [K]}\|P_U\Sigma_{k}P_U\|_{\rm F} \sqrt{ (x+\log(2TKN))} +  \max_{k \in [K]}\opnorm{P_U\Sigma_{k}P_U} (x+\log(2TKN)\right).
\end{align*}

We can deduce from the last two displays, with the same probability
\begin{align}
    \max_{1\leq t\leq T}\max_{\x_{t,n}\in \mathcal{D}_{t,n}} \|P_U(\x_{t,n})\|_2^2 \leq C \max_{k \in [K]}\{\opnorm{\Sigma_k}\}\left( \bar{r} \bigvee (x+\log(2TKN))\right), 
\end{align}
for some constant $C>0$ that can depend only $C_{\mathbf{z}}$.
\end{lemma}

\begin{proof}[Proof of Lemma \ref{lem:norm-arms}]
In view of Assumption \ref{Ass:BoundedNorms}, we have $\|P_U(\x_{t,n})\|_2^2 = \mathbf{z}_{t,n}^\top P_U\Sigma_{k}P_U\mathbf{z}_{t,n}$ for some $k\in [K]$. We apply now the Hanson-Wright's inequality conditionally on $\Fbar_{n-1}$ to get for any $x>0$
\begin{align*}
    &\mathbb{P}\left(   \mathbf{z}_{t,n}^\top \Sigma_{k}\mathbf{z}_{t,n} \leq \mathbb{E}\left[  \mathbf{z}_{t,n}^\top P_U\Sigma_{k}P_U\mathbf{z}_{t,n}   \big \vert \Fbar_{n-1} \right]  + C_{\mathbf{z}}\left(\sqrt{\|P_U\Sigma_{k}P_U\|_{\rm F}^2 x} + \opnorm{P_U\Sigma_{k}P_U}\ x\right)   \big\vert \Fbar_{n-1} \right) \geq 1 - e^{-x}.
\end{align*}
Note that $\mathbb{E}\left[  \mathbf{z}_{t,n}^\top P_U\Sigma_{k}P_U\mathbf{z}_{t,n}   \big \vert \Fbar_{n-1} \right]  = \mathrm{trace}(P_U\Sigma_{k}P_U)$ since $\mathbf{z}_k$ is a zero mean isotropic random vector. Define the event $\Omega_n \in \Fbar_{n}$ as
\begin{align*}
\Omega_n &= \bigcap_{t=1}^{T}\bigcap_{k=1}^{K} \biggl\lbrace \mathbf{z}_{t,n}^\top P_U\Sigma_{k}P_U \mathbf{z}_{t,n} \leq \mathrm{trace}(P_U\Sigma_{k}P_U)\\
&\hspace{1cm}+ C_{\mathbf{z}}\left(\sqrt{\|P_U\Sigma_{k}P_U\|_{\rm F} (x+\log(2TKN))} + \opnorm{P_U\Sigma_{k}P_U}\ (x+\log(2TKN)\right)  \biggr\rbrace
\end{align*}
Combining the last two displays with a simple union bound gives for any $n\in [N]$ and any $x>0$,
\begin{align}
\label{eq:omeganinterm1-bis}
    \mathbb{P}\left(   \Omega_n     \big\vert \Fbar_{n-1}  \right) \geq 1-\frac{1}{2N}e^{-x}.
\end{align}

We define
$$
\bar{\Omega}_{N} = \bigcap_{n=0}^{N} \Omega_n \in \Fbar_{{N}},
$$
where $\Omega_0 = \Omega$ (the sample space). The Bayes rule combined with \eqref{eq:omeganinterm1-bis} gives
\begin{align}
    \mathbb{P}\left(   \bar{\Omega}_{N} \right) = \prod_{n=1}^{N}    \mathbb{P}\left( \Omega_n \big\vert \bigcap_{k=0}^{n-1} \Omega_k  \right) \geq \left( 1-\frac{1}{2N}e^{-x} \right)^{N}.
\end{align}
Bernoulli's inequality ($(1+x)^n\geq 1+nx$ for any $x>-1$ and integer $n\geq 1$) gives 
\begin{align}
\label{eq:OmegabarN}
    \mathbb{P}\left(   \bar{\Omega}_{N} \right) = \prod_{n=1}^{N}    \mathbb{P}\left( \Omega_n \big\vert \bigcap_{k=0}^{n-1} \Omega_k  \right) \geq 1-\frac{1}{2}e^{-x}.
\end{align}
\end{proof}

\begin{lemma}\label{Le:QuadraticForm}
Considering a sequence of random covariates $\left(\x_s\right)_{s=1}^n\in\reals^d$ satisfying Assumption \ref{Ass:BoundedNorms}, where for any given policy $\pi$, $\x_s$ is $\Fcal_{s-1}$-measurable. Let $\A\in\reals^{d\times \tilde{r}}$ and $\B\in\reals^{d\times r}$ be matrices with orthonormal columns. Then, for any $x>0$, we have with probability at least $1-e^{-x}$, for any $n\in [N]$
\begin{align*}
&\opnorm{\A^\top\Sigmahat_n\B - \A^\top\Sigmabf_n\B}\\
&\hspace{0.5cm}\lesssim C_{\mathbf{z}}^2  \max_{1\leq k \leq K }\left\lbrace\opnorm{\Sigma_{k}}\right\rbrace \left(\sqrt{\frac{\max\{r,\tilde{r}\}}{n}} \bigvee \sqrt{\frac{x+ \log(N)}{n}} \bigvee \frac{\max\{r,\tilde{r}\}}{n} \bigvee \frac{x+ \log(N)}{n} \right).
\end{align*}

\end{lemma}
\begin{proof}
By definition of $\Sigmahat_n$ and $\Sigmabf_n$, we have
\begin{align*}
     &\opnorm{\A^\top\Sigmahat_n\B - \A^\top\Sigmabf_n\B} = \sup_{\ubf \in\Sb^{\tilde{r}-1},\vbf\in\Sb^{r-1}} \left\lbrace \ubf^\top \left(\A^\top\Sigmahat_n\B - \A^\top\Sigmabf_n\B\right)\vbf\right\rbrace .
\end{align*}

The first step of the proof is a standard $\epsilon$-net approximation argument. Fix $\epsilon\in (0,1/2)$. An $\epsilon$-net $\mathcal{N}_\epsilon$ of $\Sb^{m}$ is a subset of $\Sb^{m}$ such that for any $\ubf\in \Sb^{m}$, there exists $\bar{\ubf}\in \mathcal{N}_\epsilon$ such that $\|\ubf-\bar{\ubf}\|\leq \epsilon$. 
Corollary 4.2.13 in \cite{vershynin2019} guarantees the existence of an $\epsilon$-net $\mathcal{N}_\epsilon$ of $\Sb^{\tilde{r}-1}$ and an $\epsilon$-net $\mathcal{M}_\epsilon$ of $\Sb^{\bar{r}-1}$ such that
\begin{align}
    \label{eq:Nepscard}
    |\mathcal{N}_\epsilon| \leq \left(1 + \frac{2}{\epsilon} \right)^{\tilde{r}},\quad |\mathcal{M}_\epsilon| \leq \left(1 + \frac{2}{\epsilon} \right)^{\bar{r}},
\end{align}
such that
\begin{align}
    \label{eq:quadapprox}
    \sup_{\ubf \in\Sb^{\tilde{r}-1},\vbf\in\Sb^{r-1}} \left\lbrace \ubf^\top \left(\A^\top\Sigmahat_n\B - \A^\top\Sigmabf_n\B\right)\vbf\right\rbrace &\leq \frac{1}{1-2\epsilon} \max_{\bu\in \mathcal{N}_\epsilon,\bv\in \mathcal{M}_\epsilon }\left\lbrace \ubf^\top \left(\A^\top\Sigmahat_n\B - \A^\top\Sigmabf_n\B\right)\vbf\right\rbrace.
\end{align}

\noindent We now fix $\bu\in \mathcal{N}_\epsilon$, $\bv\in \mathcal{M}_\epsilon$ and we define $( \mathbf{M}_n)_{n\geq 0}$ as follows. Let $ \mathbf{M}_0=0$ a.s. and for any $n\geq 1$
$$
 \mathbf{M}_n:=  \frac{1}{n}\sum_{s=1}^n \left( \langle \ubf, \Abf^{\top}\x_{T+1,s}\rangle \langle \Bbf^{\top}\x_{T+1,s},\vbf \rangle-  \mathbb{E}\left[ \langle \ubf, \Abf^{\top}\x_{T+1,s}\rangle \langle \Bbf^{\top}\x_{T+1,s},\vbf \rangle \bigg\vert \Fbar_{s-1}\right]\right).
$$
From now on, we drop the subscript $T+1$ for the sake of brevity. By construction $\left(\M_n\right)_{n\geq 0}$ is a square-root integrable martingale adapted to $\Fcal_{n-1}$. 
We then define its increasing process by $\langle\M\rangle_0$. and for all $n\geq 1$ we have
\begin{align*}
    \langle\M\rangle_n = \sum_{s=1}^n \E\left[\left(\M_s - \M_{s-1}\right)^2 |\Fcal_{s-1}\right].
\end{align*}
We define for each $n\geq 1$,
\begin{align*}
    V_n &= \E\left[ \left( \M_n - \M_{n-1} \right)^2 | \Fcal_{n-1} \right] \leq \E\left[\left( \ubf^\top \A^\top \x_n\x_n^\top \B \vbf - \ubf^\top \A^\top \E\left[\x_n\x_n^\top | \Fcal_{n-1}\right] \B \vbf \right)^2 | \Fcal_{n-1} \right]\\
    &\leq \E\left[\left( \ubf^\top \A^\top \x_n\x_n^\top \B \vbf \right)^2 | \Fcal_{n-1} \right] - \left(\ubf^\top \A^\top \E\left[\x_n\x_n^\top | \Fcal_{n-1}\right] \B \vbf \right)^2 \\
    &\leq \E\left[\left(\langle \ubf, \Abf^{\top}\x_{n}\rangle \langle \Bbf^{\top}\x_{n},\vbf \rangle \right)^2 | \Fcal_{n-1}\right].
\end{align*}
According to Assumption \ref{Ass:ArmsDistribution} and the definition of $C_{\x}$ we have
\[
    V_n \leq \E\left[\left(z_1 z_2\right)^2\right],
\]
where $z_1$ and $z_2$ are $C_{\x}$-subGaussian random variables. Hence, their product is subExponential with parameter $2C^2_{\x}$ (see \cite{vershynin2019}[Lemma 2.7.6]). That being said, its second moment is upper bounded by $C^4_{\x}$ which gives
\[
    V_n \lesssim C^4_{\x}\leq C_{\mathbf{z}}^4  \max_{k \in [K] }\opnorm{\Sigma_{k}}^2.
\]
In order to apply Bernstein's inequality \citep{bercurio}[see Theorem 3.14], we need to control the $p$-th moment of $\Delta\M_n$
\[
    \Delta\M_n := \M_n - \M_{n-1} = \ubf^\top \A^\top \left( \x_n \x_n^\top - \E\left[\x_{n}\x_{n}^\top|\Fcal_{n-1}\right] \right) \B \vbf.
\]
Specifically, we have that for any $p\geq 3$ since we are handling a subExponential random variable  with parameter $C^2_{\x}$, the following holds
\[
    \mathbb{E}\left[ \Delta\Mtilde^p_n\big\vert \Fcal_{n-1}\right] \leq (C^2_{\x} p)^p \leq (C^2_{\x} e)^p p! \leq p! \left(  C^4_{\x} \right) e^p \left( C^2_{\x}\right)^{p-2} \leq \frac{p!}{2} V_n c_2^{p-2}
\]
where we used $c_2 = e^3 C^2_{\x}$ and the second inequality holds thanks to Stirling's approximation. 

Finally, a combination of Theorem 3.14 in \cite{bercurio} and an union bound over the $N$ rounds and the $\epsilon$-nets gives the statement.

\end{proof}

\begin{lemma}\label{lem:technorm}
Let the assumptions of Theorem \ref{Th:NewBias2} be satisfied. Then we have for any $n\geq 1$ and any $x>0$, with probability at least $1-e^{-x}$,
$$
\|\Bhat^\top \X_n\etab_n\|\lesssim 
 \max_{1\leq k \leq K }\left\lbrace \opnorm{\Sigma_{k}}^{1/2}\right\rbrace \left( c_{\eta} C_{\mathbf{z}} \, (\log(4r) + x)  + \sigma\,  \,\sqrt{2\, n\,   (\log(4r) + x) }\right)\sqrt{r}.
$$
\end{lemma}

\begin{proof}
We define $ \mathbf{M}_0=0$ a.s. and for any $n\geq 1$,
\begin{align}
  \mathbf{M}_n&:= \sum_{s=1}^n \langle \vhat_j, \x_{s} \rangle  \etab_{s} \notag.
\end{align}

\noindent By nature of $\alphahat_{T+1,n}$ in  \eqref{Eq:alphahat}, given the $\Fbar_{n-1}$, we select at round $n$ the arm $\x_{n}\in \mathcal{D}_{T+1,n}$ for the $T+1$-th task with $\mathcal{D}_{T+1,n}$ independent of $\etab_{n}$. This means that
$$
\x_{n} \indep \etab_{n} \bigg\vert \Fbar_{n-1}.
$$

Consequently, under Assumption \ref{Ass:BoundedNorms} and \ref{Ass:subGaussnoise}, $(\mathbf{M}_n)_{n\geq 0}$ is a square-root integrable martingale adapted to the filtration $\{\Fbar_{n}\}_{n\geq 0}$.

\noindent Define the corresponding increasing process by: $\langle M \rangle_0$ and for all $n\geq 1$
\begin{align}
 \langle \mathbf{M}\rangle_n = \sum_{i=1}^n \mathbb{E}\left[ (\mathbf{M}_{i} - \mathbf{M}_{i-1})^2\vert \Fbar_{i-1} \right].
\end{align}
Define for any $s\geq 1$
\begin{align}
\mathbf{V}_{s} &:=  \mathbb{E}\left[ (\mathbf{M}_{s} - \mathbf{M}_{s-1})^2\vert \Fbar_{s-1} \right].\notag\\
\end{align}

\noindent Under Assumptions \ref{Ass:BoundedNorms} and \ref{Ass:subGaussnoise}, we have for any $n\geq 1$
\begin{align}
\label{eq:boundVn}
\mathbf{V}_{s} &=\mathbb{E}\left[  \langle \vhat_j, \x_{s} \rangle^2  \etab_{s}^2 \vert \Fbar_{s-1} \right] = \sigma^2 \mathbb{E}\left[  \langle \vhat_j, \x_{s} \rangle^2  \vert \Fbar_{s-1} \right] \leq  \sigma^2 \max_{1\leq k \leq K }\left\lbrace\opnorm{\Sigma_{k}}\right\rbrace. 
\end{align}

\noindent \textbf{Fact 1.} We prove for any integer $p\geq 3$ and for all $1\leq n \leq N$,
\begin{align}\label{eq:bersntein_cond}
    \mathbb{E}\left[ |\mathbf{M}_{n} - \mathbf{M}_{n-1}|^p \vert \Fbar_{n-1}\right] &\leq \frac{p!}{2}\left(\sqrt{2} c_{\eta} C_{\x}\right) ^{p-2} \sigma^2\max_{k \in [K]} \Sigmabf_{k}, \quad a.s.
\end{align}

\begin{proof}[Proof of Fact 1]
By tower property of conditional expectation, we have
\begin{align}
\label{eq:towerprop}
    \mathbb{E}\left[ |\mathbf{M}_{n} - \mathbf{M}_{n-1}|^p \vert \Fbar_{n-1}\right] &=  \mathbb{E}\biggl[   \mathbb{E}\left[ |\mathbf{M}_{n} - \mathbf{M}_{n-1}|^p \vert \Fbar_{n-1},\x_{n} \right] \bigg\vert \Fbar_{n-1}\biggr]\notag\\
    &=\mathbb{E}\biggl[   \mathbb{E}\left[ \left|   \langle \vhat_j, \x_{n} \rangle  \etab_{n}  \right|^p \vert \Fbar_{n-1},\x_{n} \right] \bigg\vert \Fbar_{n-1}\biggr]\notag\\
   &=\mathbb{E}\biggl[   \mathbb{E}\left[ \left|\etab_{n}  \right|^p \vert \Fbar_{n-1},\x_{n} \right]  |\langle \vhat_j, \x_{n} \rangle|^{p} \bigg\vert \Fbar_{n-1}\biggr] .
\end{align}
Next, since both $\etab_n$ and $\x_n$ are subgaussian,
Lemma 5 in \cite{cella2022multi} gives for any $p\geq 3$
\begin{align}
\label{eq:towerprop2}
 \mathbb{E}\left[ \left|  \langle \vhat_j, \x_{n} \rangle  \eta_{n}   \right|^p \vert \Fbar_{n-1},\x_{n}\right] &\leq \sqrt{e(p-2)!}\;\sigma^2\; c_{\eta}^{p-2} |\langle \vhat_j, \x_{n} \rangle|^{p}\; ,\quad a.s.
\end{align}
In addition, for any $p\geq 3$,
$$
  \mathbb{E}\left[ \left|  \langle \vhat_j, \x_{n} \rangle  \right|^p \vert \Fbar_{n-1}\right] \leq \sqrt{e(p-2)!} \max_{k \in [K]}  \Sigmabf_{k} (\sqrt{2}\|\x_{n}\|_{\psi_2} )^{p-2}.
$$
The last two displays yield
\begin{align}
    \mathbb{E}\left[ |\mathbf{M}_{n} - \mathbf{M}_{n-1}|^p \vert \Fbar_{n-1}\right] &\leq  e(p-2)! \sigma^2 \max_{k \in [K]}\left\lbrace  \Sigmabf_{k}\right\rbrace (\sqrt{2}c_{\eta}\|\x_{n}\|_{\psi_2} )^{p-2}.
\end{align}
Thus we proved \eqref{eq:bersntein_cond}.

\end{proof}


Set $\bar{c}= \sqrt{2} c_{\eta} C_{\x}$. In view of \eqref{eq:bersntein_cond}, Theorem 3.14 in \cite{bercurio} gives for any positive $t,s$, with probability at least $1-e^{-n t}$,
$$
\mathbb{P}\left(  \mathbf{M}_{n} > n (\bar{c} t + \sqrt{2 t s }), \langle \mathbf{M} \rangle_{n} \leq n s    \right) \leq \exp(-n t).
$$
We study next the quadratic process $\langle \mathbf{M} \rangle_{n}$. In view of \eqref{eq:boundVn}, we immediately get
$$
\langle \mathbf{M} \rangle_{n} \leq \sigma^2 \max_{1\leq k \leq K }\left\lbrace \opnorm{\Sigma_{k}}\right\rbrace n,\quad a.s.
$$
 We take $s= \sigma^2 \max_{1\leq k \leq K }\left\lbrace \opnorm{\Sigma_{k}}\right\rbrace$ and $t'=nt$. 
Combining the last two displays gives with probability at least $1-2e^{-t'}$
\begin{align}
\label{eq:boundMn}
    \bigg\vert \langle \vhat_j, \X_n\etab_{n} \rangle \bigg\vert = \bigg\vert\sum_{s=1}^n \langle \vhat_j, \x_{s} \rangle\,  \etab_{s} \bigg\vert\leq  \sqrt{2} c_{\eta} C_{\x} \, t'  + \sigma\, \max_{1\leq k \leq K }\left\lbrace \opnorm{\Sigma_{k}}^{1/2}\right\rbrace \,\sqrt{2\, n\,   t' }.
\end{align}
We define now the event 
$$
\Omega_n = \bigcap_{1\leq j \leq \rhat} \left\lbrace  \vert \langle \vhat_j, \X_n\etab_{n} \rangle \vert \leq \sqrt{2} c_{\eta} C_{\x} \, t'  + \sigma\, \max_{1\leq k \leq K }\left\lbrace \opnorm{\Sigma_{k}}^{1/2}\right\rbrace \,\sqrt{2\, n\,   t' }  \right\rbrace.
$$
An union bound gives
$$
\mathbb{P}\left(\Omega_n^c \right) \leq 2\rhat e^{-t'}.
$$
Now we set $t' = \log(2\rhat) + x$ for some $x>0$. Consequently, we obtain that $\mathbb{P}\left(\Omega_n \right)\geq 1-e^{-x}$.

We recall that $\rhat = \rho(\overline{\Pbf})\leq C r$ on an event of probability greater than $1-e^{-x}$ in view of \eqref{eq:rankWhat}.

Combining the last two observations, we get with probability at least $1-e^{-x}$
\begin{align*}
\|\Bhat^\top \X_n\etab_n\| &\lesssim \left( c_{\eta} C_{\x} \, (\log(4r) + x)  + \sigma\, \max_{1\leq k \leq K }\left\lbrace \opnorm{\Sigma_{k}}^{1/2}\right\rbrace \,\sqrt{2\, n\,   (\log(4r) + x) }\right)\sqrt{r}\\
&\lesssim  \max_{1\leq k \leq K }\left\lbrace \opnorm{\Sigma_{k}}^{1/2}\right\rbrace \left( c_{\eta} C_{\mathbf{z}} \, (\log(4r) + x)  + \sigma\,  \,\sqrt{2\, n\,   (\log(4r) + x) }\right)\sqrt{r}.
\end{align*}

\end{proof}

\end{document}